\documentclass[twocolumn,journal]{IEEEtran}
\usepackage{amssymb}
\usepackage{mathrsfs}
\usepackage{amsfonts}
\usepackage{amsmath}
\usepackage{graphicx}
\usepackage{multirow}
\usepackage{caption2}
\usepackage{float}
\usepackage{booktabs}
\usepackage{algorithm}
\usepackage{algorithmic}
\usepackage{subfigure}
\usepackage{psfrag}
\usepackage{txfonts}
\usepackage{latexsym,bm}
\usepackage{color}
\usepackage[table]{xcolor}
\usepackage{url}

\newtheorem{theorem}{Theorem}
\newtheorem{corollary}{Corollary}
\newtheorem{lemma}{Lemma}

\newtheorem{proposition}{Proposition}

\newtheorem{problem}[theorem]{Problem}

\newenvironment{proof}[1][Proof]{\textbf{#1. }}{\ \rule{0.5em}{0.5em}}%

\ifCLASSINFOpdf
\else
\fi

\begin{document}
%
\title{Generalization Performance of Empirical Risk Minimization on Over-parameterized Deep ReLU Nets }

\author{Shao-Bo Lin, Yao Wang, and~Ding-Xuan Zhou
\IEEEcompsocitemizethanks{\IEEEcompsocthanksitem S. B. Lin  and Y. Wang are  with the Center for Intelligent Decision-Making and Machine Learning, School of Management, Xi'an Jiaotong University, Xi'an 710049,   P R China.  D. X. Zhou is with
School of Mathematics and Statistics, University of Sydney,
Sydney NSW 2006, Australia. The corresponding author is Y. Wang (email:
yao.s.wang@gmail.com).}}

 \IEEEcompsoctitleabstractindextext{
\begin{abstract}
In this paper, we study the generalization performance of global minima for implementing empirical risk minimization (ERM) on over-parameterized deep ReLU nets. Using a novel deepening scheme for deep ReLU nets, we rigorously prove that there exist perfect global minima achieving almost optimal generalization error { rates} for numerous types of data under mild conditions.  Since over-parameterization is crucial to guarantee that the global minima of ERM on   deep ReLU nets can be realized by the widely used stochastic gradient descent (SGD) algorithm, our results indeed fill a gap between optimization and generalization of deep learning.
\end{abstract}

}

\maketitle

\IEEEdisplaynotcompsoctitleabstractindextext

\IEEEpeerreviewmaketitle
%
%

\section{Introduction}
Deep learning \cite{Hinton2006} that conducts   feature exaction and statistical modelling on a  unified deep neural network (deep net) framework has attracted enormous research activities in the past decade. It has made significant breakthroughs in numerous applications including computer vision \cite{Krizhevsky2012}, speech recognition \cite{Lee2009} and { go games} \cite{Silver2016}. Simultaneously , it also  brings  several  challenges in understanding the running mechanism and magic behind deep learning, among which   the generalization  issue in statistics \cite{Zhang2016} and convergence issue in optimization \cite{Allen-Zhu2018} are crucial.

The generalization issue   pursues theoretical advantages  of deep learning via presenting its better generalization capability than shallow learning such as  kernel methods   and shallow networks.
Along with the rapid development of deep nets in approximation theory  \cite{Chui1994,Mhaskar2016,Yarotsky2017,Petersen2018,Shaham2018,Schwab2018,Chui2019}, essential progress on the excellent generalization performance  of deep learning   has been made in \cite{Imaizumi2018,Bauer2019,Lin2019,Schmidt2020,Chui2020,Han2020}. { To be detailed}, \cite{Schmidt2020} proved that implementing empirical risk minimization (ERM) on {  deep  ReLU nets} is better than shallow learning in embodying the { composite}  structure of { the well known regression function \cite{Gyorfi2002}}; \cite{Han2020} showed that {  implementing} ERM on deep ReLU nets { succeeds in improving} the learning performance of shallow learning by capturing the group-structure of { inputs};  and \cite{Chui2020} derived that, with the help of massive data, implementing ERM on deep ReLU nets { is capable of  reflecting}    spatially sparse properties of the regression function  while shallow learning fails. In a word, with appropriately selected number of free parameters, the generalization issue of deep learning seems to be { successfully settled} in terms that deep learning  can achieve optimal learning rates for numerous types of data, which is beyond the capability of shallow learning.

The convergence issue  {concerns  the convergence of some popular algorithms such as the stochastic gradient descent (SGD) and adaptive moment estimation (Adam)  to solve ERM on  deep nets}.  Due to the  highly nonconvex nature, {  it is generally difficult to find global minima of such  ERM problems, since there are  numerous  local minima, saddle points, plateau and even some flat regions \cite{Goodfellow}.  It is thus highly desired to declare when and where the corresponding  SGD or Adam converges.  Unfortunately, in the under-parameterized setting exhibited in   \cite{Imaizumi2018,Bauer2019,Lin2019,Schmidt2020,Chui2020,Han2020}, the convergence issue remains open.}  Alternatively, studies in optimization theory show that over-parameterizing    deep ReLU nets { enhance }  the convergence of SGD  \cite{Du2018b,Mei2018,Allen-Zhu2018,Allen-Zhu2019}. In fact, it was proved in \cite{Allen-Zhu2018,Allen-Zhu2019}  that
SGD   converges to either a global minimum or a local minimum near some global minimum of ERM, provided there are sufficiently many free parameters in deep ReLU nets. Noting further that implementing ERM on over-parameterized deep ReLU nets commonly yields infinitely many global minima  and usually leads to over-fitting, it is difficult to verify the good generalization {  performance}  of {  the} obtained global minima {  in theory.}

As mentioned above, there is an inconsistency between generalization and convergence issues of deep learning in the sense that good generalization requires under-parameterization of deep ReLU nets while provable convergence needs over-parameterization, which makes the running mechanism of deep learning be still a mystery. Such an inconsistency stimulates to rethink the classical bias-variance trade-off in modern machine learning practice \cite{Belkin2019}, since numerical evidences in \cite{Zhang2016} {  illustrated}   that there are over-parameterized deep nets generalizing well despite they achieved  {  extremely small} training error. In particular, \cite{Belkin2018}   constructed an exact interpolation of training data  that achieves  optimal generalization error {  rates} based on   Nadaraya-Watson kernel estimates; \cite{Bartlett2020} derived a sufficient condition for the data under which the global minimum of over-parameterized  linear regression possesses excellent generalization performance; \cite{Liang2020} studied the generalization  performance of kernel-based least norm interpolation and presented  generalization error estimates under some restrictions on the data distribution. Similar results can also  be found in \cite{Belkin2018nips,Hastie2019,Muthukumar2020,Lin2020,Bartlett2021} and references therein.
{  All  these  exciting results provide a springboard to understand  benign over-fitting  in modern machine-learning practice and present novel insights in developing learning algorithms to avoid the traditional bias-variance dilemma. }
However, {  it should be pointed out that  these} existing results are incapable of {  settling}  the generalization and convergence challenges of deep learning, mainly  due to  the following three aspects:

$\bullet$ {   Difference in model:} the existing theoretical analysis for benign over-fitting is only available to convex linear models, but ERM on deep ReLU nets involves highly nonconvex nonlinear models.

$\bullet$ {  Difference in theory:   the existing  results  for benign over-fitting focus on pursuing  the restrictions on   data distributions   so that all global minima for over-parameterized linear models are good learners, which does not hold for deep learning, since  it is easy to provide  a counterexample for benign over-fitting of deep ReLU nets, even for noise-less data (see Proposition \ref{Proposition:bad-interpolation} below).}

$\bullet$ {  Difference in requirements of dimension:  Theoretical analysis in the existing work frequently requires the high dimensionality  assumption  of  the input space, while the numerical experiments in \cite{Zhang2016} showed that deep learning can lead to benign over-fitting in both high and  low dimensional  input spaces.}

{  Based on these three interesting observations, there naturally arises  the following  problem:}

\begin{problem}\label{prob:1}
Without strict restrictions on   data distributions and dimensions of input spaces, are there   global minima of ERM on over-parameterized deep ReLU nets achieving {  the optimal generalization error rates obtained by under-parameterized deep ReLU nets?}
\end{problem}

{  There are roughly two  schemes to settle the inconsistency between optimization and generalization for deep learning. One is to pursue the convergence guarantee for SGD on under-parameterized deep ReLU nets and the other is to study the generalization performance of implementing ERM on over-parameterized deep ReLU nets. To the best of our knowledge, there isn't any theoretical analysis for the former, even when strict restrictions are imposed on the data distributions.   An answer to
Problem \ref{prob:1}, a stepping-stone to demonstrate the feasibility of the latter, not only provides  solid theoretical evidences for the benign over-fitting  phenomenon of deep learning in practice \cite{Zhang2016}, but also presents theoretical guidance on setting  network structures to balance the generalization and optimization inconsistency.}

%

{  The aim of the present paper is to present an affirmative answer to Problem \ref{prob:1}.} Our main tool for analysis is a novel network deepening approach based on the localized approximation property \cite{Chui1994,Chui2020} and  product-gate property \cite{Yarotsky2017,Petersen2018} of deep nets. The network deepening approach succeeds in constructing   over-parameterized deep nets (student network) via deepening and widening an arbitrary under-parameterized network (teacher network) so that the obtained student network   exactly interpolates the  training data and possesses almost the same generalization capability as the teacher network. In this way, setting the teacher network  to be the one in \cite{Imaizumi2018,Bauer2019,Lin2019,Schmidt2020,Chui2020,Han2020}, we actually prove that there are global minima for ERM on  over-parameterized deep ReLU nets that possess  optimal generalization error {  rates}, provided that the networks are deeper and wider than the corresponding student network.

The main contributions of the paper are two folds. {  Since} the presence of noise is a crucial factor  to result in  over-fitting and it is not difficult to design learning algorithms with good generalization performance to produce  perfect fit for noiseless data \cite{Li2018,Cao2019}, our first result is then to study the generalization capability of deep ReLU nets that {  exactly interpolate noiseless data}. In particular, we construct a deep ReLU net that exactly interpolates the noiseless data but performs extremely  {  badly} in generalization and also  prove that  for deep ReLU nets with more than two hidden layers, there always exist global minima of ERM that can generalize extremely well, provided the number of free parameters achieves a certain level and the data are noiseless. This is different from the linear models studied in \cite{Belkin2018,Belkin2018nips,Belkin2019,Hastie2019,Liang2020,Muthukumar2020,Lin2020,Bartlett2021}  and shows the difficulty in analyzing over-parameterized deep ReLU nets.
Our second result,  more importantly, focuses on the existence of a perfect global minimum  of ERM on over-parameterized deep ReLU nets that achieves almost optimal generalization error {  rates} for numerous types of data. Using the network deepening approach,  we rigorously prove that, if the depth and width of a deep net are larger than specific values, then there always exists such a perfect global minimum. This finding partly demonstrates the reason of the {  benign over-fitting  phenomenon  of deep learning } and   shows that implementing ERM on over-parameterized deep nets can derive an estimator of high quality.
Different from the existing results,
our  analysis  requires neither   high dimensionality   of the input space  nor strong restrictions on the covariance matrix of the input data.

The rest of this paper is organized as
follows.   In the next section, after introducing the deep ReLU nets, we provide theoretical guarantee for the existence of good deep nets interpolant for noiseless data. In Section \ref{Sec.Noisy}, we present our main results via rigorously proving the existence of perfect global minima.
 In Section \ref{Sec.Related-work}, we compare our results with related work and present some discussions. In Section \ref{Sec.numerical}, we conduct  numerical experiments to  verify our theoretical assertions.    We prove our results in the last section.

\section{Global Minima of ERM on Over-Parameterized Deep Nets for Noiseless Data}\label{Sec.Noiseless}
Let $L\in\mathbb N$ be the depth of a deep net, $d_0=0$ and $d_\ell \in \mathbb{N}$ be the width of the $\ell$-th hidden layer for $\ell=1,\dots,L-1$. Denote by the affine operator $\mathcal J_\ell:\mathbb R^{d_{\ell-1}}\rightarrow\mathbb R^{d_\ell}$ with $\mathcal J_\ell(x):=W_\ell  x+b_\ell$ for $d_\ell\times d_{\ell-1}$ weight matrix $W_\ell$ and  bias vector  $b_\ell\in\mathbb R^{d_\ell}$. For  the ReLU function $\sigma(t)=t_+:=\max\{t,0\}$, write  $\sigma(x)=(\sigma(x^{(1)}),\dots,\sigma(x^{(d)}))^T$ for $x=(x^{(1)},\dots,x^{(d)})^T$.
Define  an $L$-layer deep {  ReLU} net  by
\begin{equation}\label{deep-net}
     \mathcal N_{d_1,\dots,d_L}(x)
     = a\cdot \sigma\circ \mathcal J_L \circ \sigma\circ \mathcal J_{L-1} \circ \dots \circ \sigma\circ\mathcal J_1(x),
\end{equation}
where $a \in\mathbb R^{d_L}$. The structure of $ \mathcal N_{d_1,\dots,d_L}$ is determined by   the weight matrices $W_\ell$ and bias vectors $b_\ell$, $\ell=1,\dots,L$. In particular, full weight matrices correspond to deep fully connected nets (DFCN) \cite{Yarotsky2017}; sparse weight matrices are associated with
 deep   sparsely connected nets (DSCN) \cite{Petersen2018}; and
 Toeplitz-type weight matrices are related to deep convolutional neural networks (DCNN) \cite{Zhou2018}.
 In particular, for DFCN, we have {  the number of training parameters}
\begin{equation}\label{Num-param}
      n=d_L+\sum_{\ell=1}^L (d_{\ell-1}d_\ell+d_{\ell})
\end{equation}
and for   DSCN, $n$ is much smaller  than the number in \eqref{Num-param}. {  In this paper, we mainly focus on analyzing the benign over-fitting phenomenon for DFCN. Denote by $\mathcal N_{d_1,\dots,d_L}^{DFCN}$ the set of all DFCNs  of the form  \eqref{deep-net}.  Define the width of DFCN to be $U:=\max\{d,d_1,\dots,d_L \}.$}

Given {  a} sample set $D=\{(x_i,y_i)_{i=1}^m\}$,  we  are interested in    global minima of the following empirical risk minimization (ERM):
\begin{equation}\label{target-optimization}
    {\arg\min}_{f\in\mathcal N_{d_1,\dots,d_L}^{DFCN}}\frac1m\sum_{i=1}^m(f(x_i)-y_i)^2.
\end{equation}
Denote by $\Psi_{{d_1,\dots,d_L},m}$ the set of  global minima of the optimization problem (\ref{target-optimization}), i.e.,
\begin{equation}\label{global-minima-set}
\Psi_{ d_1,\dots,d_L,m}:= \{f: f\ \mbox{is the solution to}\ (\ref{target-optimization})\}.
\end{equation}
{  Before studying  the quality of the global minima of \eqref{target-optimization}, we present some properties  of $\Psi_{d_1,\dots,d_L,m}$ in the following lemma.}

{
\begin{lemma}\label{Lemma:massiveness}
If $d_1\geq m$, then for any $L\geq 1$ and $d_2,\dots,d_L\geq 2$, there are infinitely many functions in $\Psi_{ d_1,\dots,d_L,m}$ and for any $f\in \Psi_{ d_1,\dots,d_L,m}$, there holds
  $f(x_i)=y_i, i=1,\dots,m$.
\end{lemma}
}

{   Lemma \ref{Lemma:massiveness} is a direct extension of \cite[Theorem 5.1]{Pinkus1999}, where similar conclusion is drawn for $\mathcal N_{d_1}^{DFCN}$ with $d_1\geq m$. Based on  \cite[Theorem 5.1]{Pinkus1999}, the proof of Lemma \ref{Lemma:massiveness} is obvious by noting $t=\sigma(t)-\sigma(-t)$ and $d_2,\dots,d_L\geq 2$.
  Lemma \ref{Lemma:massiveness} implies that if $m$ is sufficiently large and the network structure satisfying  $d_1\geq m$, then there are always infinitely many solutions to  (\ref{target-optimization}) and any global minimum $f$ exactly interpolates the given data. It should be mentioned that   similar results  also hold  for any DSCNs that contain  a shallow net with $m$ neurons. Since we place particular emphasis on DFCN, we leave  the corresponding assertions for DSCN for interested readers.  }


Let $\mathbb I^d:=[-1,1]^d$. Denote by $L^p(\mathbb I^d)$ the space of $p$th-Lebesgue integrable functions endowed with norm $\|\cdot\|_{L^p(\mathbb I^d)}$.  In this section, we are concerned with   noiseless data, that is,
 there is some $f^*\in L^p(\mathbb I^d)$  such that
\begin{equation}\label{noiseless-setting}
    f^*(x_i)=y_i,\qquad i=1,\dots,m.
\end{equation}
At first, we  provide
  a negative result for running ERM on over-parameterized deep ReLU nets, {  showing that there is some $f\in \Psi_{{d_1,\dots,d_L},m}$ performs extremely badly in generalization.}

\begin{proposition}\label{Proposition:bad-interpolation}
{  Let $1\leq p<\infty$. If $L\geq 2$, $d_1\geq 4dm$, $d_2\geq m$ and $d_3,\dots, d_{L}\geq 2$, then there are infinitely many $ f \in\Psi_{d_1,d_2,\dots,d_L,m}$ } such that   for any  $f^*$ satisfying (\ref{noiseless-setting}) and  $\|f^*\|_{L^p(\mathbb I^d)}\geq c$  there holds
 \begin{equation}\label{interpolation-5}
       \|f^*-f \|_{L^p(\mathbb I^d)}\geq c/2,
\end{equation}
where $c$ is an absolute constant.
\end{proposition}

Proposition \ref{Proposition:bad-interpolation} shows that {  in} the  over-parameterized setting,  {  there are infinitely many global minima of (\ref{target-optimization}) behaving extremely badly}, even for the noiseless data. It should be highlighted that the construction of $f$ in our proof is independent of $f^*$. In fact, due to the localized approximation of deep ReLU nets, we can construct deep ReLU nets $f$ which can exactly interpolate the training sample but {  satisfy} $\|f\|_{L^p(\mathbb I^d)}\leq \varepsilon$ for arbitrarily small $\varepsilon>0$.
 Different from the results in \cite{Belkin2019,Hastie2019,Muthukumar2020,Liang2020}, it is difficult to quantify   conditions on distributions and dimensions of   input spaces such that all global minima of over-parameterized deep ReLU nets   possess good generalization performances. This is mainly due to the nonlinear nature of the hypothesis space, though the capacities, measured by  the covering numbers \cite{Guo2019,Bartlett2019}, for deep ReLU nets and linear models are comparable, provided there are similar {  numbers} of free parameters involved in hypothesis spaces and the depth is not so large.

Proposition \ref{Proposition:bad-interpolation} provides extremely  bad examples for global minima of \eqref{target-optimization} in generalization, even for noiseless data. It seems that over-parameterized deep ReLU nets are always worse than under-parameterized networks, which is standard from {  a} viewpoint of classical learning theory \cite{Cucker2007}. However, in our following theorem,
we will show that if the number of free parameters increases to a certain extent, then there are also infinitely many global minima of (\ref{target-optimization}) possessing excellent generalization performance. To this end, we need two concepts concerning the data distribution and the smoothness of the target function $f^*$.

Denote $\Lambda:=\{x_i\}_{i=1}^m$ as the input set.  The separation radius \cite{Narcowich2004} of $\Lambda$ is defined by
\begin{equation}\label{separ-radius}
     q_\Lambda=\frac12\min_{i\neq j}\|x_i-x_j\|_2,
\end{equation}
where $\|x\|_2$ denotes the Euclidean norm of
  $x\in\mathbb R^d$.
The separation radius is  half of the smallest distance between any two distinct points in $\Lambda$ and naturally satisfies $q_\Lambda\leq m^{-1/d}$.
Let us  also introduce the standard smoothness assumption  \cite{Gyorfi2002,Yarotsky2017,Petersen2018,Lin2019,Zhou2020b}.
 Let   $c_0>0$ and
$r=s+\mu$ with $s\in\mathbb N_0:=\{0\}\cup\mathbb N$ and $0<\mu\leq 1$.
We say a   function $f:\mathcal A\subseteq\mathbb R^d\rightarrow\mathbb R$ is
$(r,c_0)$-smooth if $f$ is $s$-times differentiable and for every
$\alpha_j\in \mathbb N_0$, $j=1,\dots,d$ with
$\alpha_1+\dots+\alpha_d=s$, its $s$-th partial derivative satisfies
the Lipschitz condition
\begin{equation}\label{lip}
          \left|\frac{\partial^sf}{\partial x_1^{\alpha_1}\dots\partial
          x_d^{\alpha_d}}
          (x)-\frac{\partial^sf}{\partial x_1^{\alpha_1}\dots\partial
          x_d^{\alpha_d}}
          (x')\right|\leq c_0\|x-x'\|_2^\mu,\quad \forall\ x,x'\in\mathcal A.
\end{equation}
Denote by $Lip_\mathcal A^{(r,c_0)}$ the set of all
$(r,c_0)$-smooth functions defined on $\mathcal A$.

We are in a position to state our first main result to show the existence of good global minima of (\ref{target-optimization}) for noiseless data.

\begin{theorem}\label{Theorem:interpolation-app-linear}
{  Let $r,c_0>0$ and $N\in\mathbb N$. If $f^*\in Lip_{\mathbb I^d}^{(r,c_0)}$ satisfies (\ref{noiseless-setting}),  $N\succeq q_\Lambda^{-d}$,
$L\succeq \log N$,  $d_1\succeq N$ and $d_\ell\succeq \log N$ for $\ell=2,\dots,L$, then there are infinitely many    $h^*\in\Psi_{d_1,\dots,d_L,m}$,}
such that
\begin{equation}\label{inter-jackson}
     \|h^*-f^*\|_{L^\infty(\mathbb I^d)}\leq CN^{-r/d},
\end{equation}
where  $C>0$ is a constant  depending only on $r,d$ and $c_0$,
{  and $a\succeq b$ for $a,b>0$ means that there is some constant $C'$ depending only on $r,d$ and $c_0$ such that $a\geq C'b$.}
\end{theorem}

{
A   consensus on deep nets approximation is that it can break the ``curse
of dimensionality'', which was verified in interesting work \cite{Mhaskar2016,Kohler2017,Shaham2018,Kohler2019,Chui2019,Schmidt2020} in terms of deriving dimension-independent approximation rates. However, it should be pointed out that  to achieve such   dimension-independent
approximation rates, strict restrictions have been imposed on target functions, which become stronger as the dimension $d$ grows, just as \cite[P.68]{Barron2008} observed. In this way, though the
approximation error of deep nets is independent of the dimension, the
applicable  target functions become  more and more stringent as $d$
grows. Our result presented in Theorem \ref{Theorem:interpolation-app-linear} drives a different direction to show that even for well-studied smooth functions, there is a deep net that interpolates the training data without degrading the approximation rate. We highlight that the approximation rate depends on the a-priori knowledge of the target functions. In particular, if we impose strict restriction  such as $f^*\in Lip_{\mathbb I^d}^{(r,c_0)}$ with  $r\geq d/2$, which is standard in kernel learning \cite{Caponnetto2007}, then the approximation rate can be at least of order $n^{-1/2}$, which is also dimension-independent.
}

{
It is well known that deep learning has achieved great success in applications whose input space are of high  dimensionality  such as image processing \cite{Krizhevsky2012} and game theory \cite{Silver2016}, showing the excellent performance of deep learning with large $d$. However, recent progress in inventory management \cite{Qi2022}, finance prediction \cite{Hong2022} and earthquake intensity analysis \cite{Han2020} demonstrated that deep learning is also efficient for applications with low-dimensional input spaces. Therefore, numerous research activities have been triggered to verify the advantage of deep nets from approximation theory \cite{Yarotsky2017,Petersen2018,Zhou2020a} and learning theory \cite{Imaizumi2018,Kohler2019,Chui2020} that regarded $d$ as a constant. This paper follows from this direction to  assume $d$ to be a constant and is much smaller than the size of data. As we mentioned above, if $d$ is extremely large, more restrictions on the target functions should be imposed, just as Theorem \ref{Theorem:additive-learning} below purports to show.
}

   Noting that there are totally $\mathcal O(N\log N)$ parameters in $h^*$ in \eqref{inter-jackson}, the derived error rate is almost  optimal in the sense that up to a logarithmic factor, the derived upper bound is of the same order of lower bound \cite{Yarotsky2017,Guo2019}, i.e.,
 $C_1'(N\log N)^{-r/d}$ for some $C_1'>0$ independent of $N$. 
  This means that for  {  $N\succeq q_\Lambda^{-d}$}  and
{  $ L\succeq \log N $,} we can get {  almost}  optimal deep nets via finding   suitable global minima of (\ref{target-optimization}).  Theorem \ref{Theorem:interpolation-app-linear} also shows that in the over-parameterized setting, where all global minima exactly interpolate the data, {   the interpolation restriction does not always affect} the approximation performance of deep nets.
Theorem \ref{Theorem:interpolation-app-linear}  actually presents a sufficient condition for the number of free parameters  of deep ReLU nets , {  $N\succeq q_\Lambda^{-d}$}, to guarantee the existence of perfect global minima when we are faced with noiseless data. It should be highlighted that $q_\Lambda$ can be numerically  determined, provided the data set is given. If  $\{x_i\}_{i=1}^m$ are drawn identically and independently (i.i.d.) according to some distribution,  the lower bound   of $q_\Lambda$ is  easy to be  derived in theory \cite{Gyorfi2002,Liang2020,Lin2020}.

\section{Global Minima of ERM on Over-Parameterized Deep Nets for Noisy Data}\label{Sec.Noisy}
In this section, we conduct our study in the standard least-square regression framework \cite{Gyorfi2002,Cucker2007} and assume
\begin{equation}\label{noisy-setting}
     y_i=f^*(x_i)+\varepsilon_i,\qquad i=1,\dots, m,
\end{equation}
where $\{x_i\}_{i=1}^m$ are i.i.d. drawn according to an unknown distribution $\rho_X$ on an input space $\mathcal X\subseteq[0,1]^d$ and $\{\varepsilon_i\}_{i=1}^m$ are independent random variables that are  independent  of $\{x_i\}_{i=1}^m$ and {  satisfy} $|\varepsilon_i|\leq 1$, $\mathbf E[\varepsilon_i]=0$, $i=1,\dots,m.$
   {   Denote by $\Xi$ a set of  Borel probability measures on $\mathcal X$ and   $\Theta$  a set of functions defined on $\mathcal X$, respectively. Write
$$
   \mathcal M(\Theta,\Xi):=\{(\rho_X,f^*):\rho_X\in\Xi, f^*\in\Theta\}.
$$
}
{  Let $\Gamma_D$ be the  class of all functions that are  derived  based on the data set   $D$.}
  Define
\begin{eqnarray}\label{ideal-quantity}
          e(\Theta,\Xi)
           := \sup_{{  (\rho_X,f^*)}\in \mathcal M(\Theta,\Xi)}\inf_{f_D\in \Gamma_D}\mathbf E(\|f^*-f_{D}\|^2_{L^2_{\rho_X}}).
\end{eqnarray}
{   It is easy to see that $e(\Theta,\Xi)$ measures the theoretically optimal generalization bound  a learning scheme, based on data set $D$ satisfying \eqref{noisy-setting},   can achieve when  $f^*\in\Theta$ and   $\rho_X\in\Xi$ \cite{Gyorfi2002}.}
Our purpose is  to   compare  generalization errors of the global minima defined by (\ref{target-optimization}) with $ e(\Theta,\Xi)$ to illustrate  whether the global minima can achieve  the theoretically optimal generalization error bounds.

Before presenting {  our main results,  we   introduce a negative result  derived in \cite[Theorem 2]{Kohler2019}, which
  shows that without  any   restriction imposed to   the distribution $\rho_X$,  over-parameterized deep nets do not generalize well.
\begin{lemma}\label{Lemma:bad-generalization}
If $d_1\geq m$, then for any $L\geq 1$ and $d_2,\dots,d_L\geq 2$ and any $ h\in \Psi_{ d_1,\dots,d_L,m}$,  there exists a probability measure $\rho_X^*$ on $\mathcal X$ such that
\begin{equation}\label{bad-generalziation}
   \sup_{f^*\in Lip_{\mathbb I^d}^{(r,c_0)},\rho_X=\rho_X^*}\mathbf E[\|f^*-h\|^2_{L^2_{\rho_{X}^*}}]
   \geq 1/6.
\end{equation}
\end{lemma}
}

{  Lemma \ref{Lemma:bad-generalization}} shows that without any restrictions to the distribution, even for the widely studied smooth functions, all of the global minima of (\ref{target-optimization}) in the over-parameterized setting are bad estimators. This is totally different from the under-parameterized setting, in which distribution-free optimal generalization error {  rates} were established \cite{Gyorfi2002,Maiorov2006a}.
{  Lemma \ref{Lemma:bad-generalization}} seems to contradict with the numerical results in \cite{Zhang2016} at the first glance. However, we highlight  that $\rho_X^*$ {  in \eqref{bad-generalziation}}  is a very special distribution that is even not continuous with respect to the Lebesgue measure. If we impose some mild conditions to exclude such distributions, then the result will be totally different.

 The restriction we study in this paper is the {  following well known  distortion assumption of $\rho_X$  \cite{Shi2013,Chui2020}, which is slightly stricter than the standard assumption that $\rho_X$ is absolutely continuous with respect to the Lebesgue measure. }
Let $p\geq 2$ and $J_p$ be the identity mapping
$$
       L^p(\mathcal X)       ~~ {\stackrel{J_p}{\longrightarrow}}~~  L^2_{\rho_X}.
$$
Define $D_{\rho _{X},p}:=\|J_p\|,$ where $\|\cdot\|$ denotes the operator norm.   Then $D_{\rho _{X},p}$ is called the
distortion of $\rho _{X}$ (with respect to the Lebesgue measure), which measures how much $\rho _{X}$ distorts the
Lebesgue measure. In our analysis, we assume $D_{\rho _{X},p}<\infty$, which holds for
 the uniform
distribution for all $p\geq2$ obviously. Denote by $\Xi_p$ the set of $\rho_X$ satisfying $D_{\rho _{X},p}<\infty$.

We then provide optimal generalization error {  rates} for   global minima of \eqref{target-optimization} for different a-priori information.
Let us begin with the widely used class of smooth regression functions. The classical results in \cite{Gyorfi2002} showed that
\begin{equation}\label{best-smooth}
       C_1 m^{-\frac{2r}{2r+d}} \leq
       e(Lip^{(r,c_0)}_{\mathbb I^d},\Xi_p)
         \leq C_2
     {m}^{-\frac{2r}{2r+d}}, \qquad p\geq 2,
\end{equation}
where $C_1$, $C_2$ are constants independent of $m$ and $e(Lip^{(r,c_0)}_{\mathbb I^d},\Xi_p)$ is defined by \eqref{ideal-quantity}. It demonstrates the optimal generalization error {  rates} for $f^*\in Lip^{(r,c_0)}_{\mathbb I^d}$ and $\rho_X\in\Xi_p$ that a good learning algorithm should achieve. Furthermore, it can be found in \cite{Schmidt2020,Han2020} that there is some network structure $\Phi_{n,L}$ such that all global minima of  ERM on under-parameterized deep ReLU nets can   reach almost optimal error {  rates} in the sense that if {  $L\sim \log m$, $d_1\sim m^{\frac{d}{2r+d}}$ and $d_2,\dots, d_L\sim \log m$, then}
\begin{equation}\label{smooth-optimal-rate}
       C_1 m^{-\frac{2r}{2r+d}}
        \leq \sup_{f^*\in Lip^{(r,c_0)}_{\mathbb I^d},\rho_X\in\Xi_p}  \mathbf E[\|f^{under}_{global}-f^*\|_{L_{\rho_X^2}}^2]
     \leq C_3
     \left(\frac{m}{\log m}\right)^{-\frac{2r}{2r+d}},
\end{equation}
where $C_3$ is a constant depending only on $r,c_0, d,p$ and $f^{under}_{global}$ is an arbitrary global minimum of (\ref{target-optimization}) {  with depth and width specified  as above.}

It follows from  (\ref{smooth-optimal-rate}) and (\ref{best-smooth}) that all global minima of ERM on under-parameterized deep ReLU nets are almost optimal learners to tackle  data i.i.d. drawn according to $\rho\in \mathcal M(Lip^{(r,c_0)},\Xi_p)$. In the following theorem, we show that, {  if the network is deepened and widened}, there also exist  global minima of (\ref{target-optimization}) for over-parameterized deep ReLU nets possessing similar generalization performance.
 \begin{theorem}\label{Theorem:smooth-learning}
Let $p\geq 2$, $r,c_0>0$. {  If $L\succeq \log m$, $d_1\succeq m$ and $d_2,\dots,d_L\succeq \log m$, then
 there exist infinitely many $h\in\Psi_{d_1,\dots,d_L,m}$ } such that
\begin{equation}\label{good-global-smooth}
       C_1 m^{-\frac{2r}{2r+d}} \leq
      \sup_{f^*\in Lip^{(r,c_0)}_{\mathbb I^d},\rho_X\in\Xi_p} \mathbf E[\|h-f^*\|_{L_{\rho_X}^2}^2]
     \leq 2C_3
     \left(\frac{m}{\log m}\right)^{-\frac{2r}{2r+d}}.
\end{equation}
\end{theorem}

{  Due to Lemma \ref{Lemma:massiveness}, any $h\in\Psi_{d_1,\dots,d_L,m}$ is an
exact interpolant of $D$, i.e. $h(x_i)=y_i$, $i=1,\dots,m$. Therefore,} Theorem \ref{Theorem:smooth-learning} presents  theoretical verifications of benign over-fitting for deep ReLU nets, which  was intensively discussed recently \cite{Belkin2018,Belkin2019,Liang2020,Lin2020,Bartlett2020}. It should be mentioned that {  our main novelties are two folds. On   one hand, we are concerned with deep ReLU nets while the existing results focused on linear hypothesis spaces. On the other hand,  we provide   evidence of existence of good global minima without high-dimensionality and strong distribution assumptions while the existing results focused on searching strong conditions on the dimension of input spaces and data distributions to guarantee that all global minima have excellent generalization performances. }
{  Theorem \ref{Theorem:smooth-learning} shows  that implementing ERM on over-parameterized deep ReLU nets} can achieve the almost optimal generalization error {  rates}, but it does not demonstrate the power of depth since    shallow learning also reaches these bounds \cite{Caponnetto2007,Maiorov2006a}. To show the power of depth, we should impose more restrictions on regression functions. For this purpose, we introduce the generalized additive models  that are widely used in statistics and machine learning \cite{Gyorfi2002,Schmidt2020}. For $r,\gamma,c_0,c_0'>0$, we say that $f$ admits a generalized additive model
if
$
       f=h\left(\sum_{i=1}^df_i(x^{(i)})\right),
$
where $h\in Lip^{(r,c_0)}_{\mathbb R}$ and $f_{i}\in Lip^{(\gamma,c_0')}_{\mathbb I}$. Write $\mathcal W_{r,\gamma,c_0,c_0'}$ as the set
of all functions  {  admitting}  a generalized additive model. {  If $L\sim \log m$, $d_1\sim  m^{\frac1{2r+1}}$ and $d_2,\dots,d_L \sim \log m$, it can be found in \cite{Schmidt2020,Han2020} that for any $p\geq 2$, there holds
\begin{eqnarray}\label{additive-1}
       && C_4 \left(m^{-\frac{2r}{2r+1}}+m^{-\frac{2\gamma(r\wedge1)}{2\gamma(r\wedge1)+1}}\right) \leq
       e(\mathcal W_{r,\gamma,c_0,c_0'},\Xi_p) \nonumber\\
         &\leq& 
        \sup_{f^*\in \mathcal W_{r,\gamma,c_0,c_0'},\rho_X\in\Xi_p}
        \mathbf E[\|f_{global}^{under}-f^*\|_{L^2_{\rho_X}}^2]\nonumber\\
      &\leq&  C_5
     \left(m^{-\frac{2r}{2r+1}}+m^{-\frac{2\gamma(r\wedge1)}{2\gamma(r\wedge1)+1}}\right)\log^3m, 
\end{eqnarray}
where   $C_4$, $C_5$ are constants
 independent of $m$ and $f_{global}^{under}$ is an arbitrary global minimum of (\ref{target-optimization}). }
Noticing that shallow learning is difficult to achieve the above  generalization error {  rates}: even for a special case {  of the} generalized additive model with $f_i(x^{(i)})=(x^{(i)})^2$, it has been proved in \cite{Chui2019} that shallow nets with any activation functions cannot achieve the aforementioned generalization error {  rates}.    The following theorem presents the existence of perfect global minima of (\ref{target-optimization}) to show the power of depth in the over-parameterized setting.

\begin{theorem}\label{Theorem:additive-learning}
Let $p\geq 2$ and $r,\gamma,c_0,c_0'>0$. {  If  $L\succeq\log m$, $d_1\succeq m$, and $d_2,\dots,d_L\succeq\log m$, then there are infinitely many $h\in\Psi_{d_1,\dots,d_L,m}$    such that
\begin{eqnarray}\label{good-global-additive}
       &&C_4 \left(m^{-\frac{2r}{2r+1}}+m^{-\frac{2\gamma(r\wedge1)}{2\gamma(r\wedge1)+1}}\right)
       \leq
      \sup_{f^*\in \mathcal W_{r,\gamma,c_0,c_0'},\rho_X\in\Xi_p} \mathbf E[\|h-f^*\|_{L_{\rho_X}^2}^2] \nonumber\\
     &\leq&  2C_5
     \left(m^{-\frac{2r}{2r+1}}+m^{-\frac{2\gamma(r\wedge1)}{2\gamma(r\wedge1)+1}}\right)\log^3m.  
\end{eqnarray}
}
\end{theorem}


Theorem \ref{Theorem:additive-learning} shows that there are infinitely many global minima of ERM on over-parameterized deep ReLU nets   theoretically  breakthroughing the bottleneck of shallow learning. This illustrates {  an} advantage of adopting over-parameterized deep ReLU nets to build up   hypothesis spaces in practice.

Finally, we show the power of depth in capturing the widely used spatially sparse features with the help of massive data.  {  It has been discussed in \cite{Chui2020,Liu2022} that spatial  sparseness  is an important data feature  for image and signal processing and deep ReLU nets perform excellently in reflecting the spatial sparseness.}
  Partition $\mathbb
I^d$ by $(N^*)^d$ sub-cubes $\{A_j\}_{j=1}^{(N^*)^d}$ of side length
$(N^*){-1}$ and with centers $\{\zeta_j\}_{j=1}^{(N^*)^d}$. For $u\in\mathbb
N$ with $u\leq (N^*)^d$, define
$$
    \Upsilon_u:=\left\{j_\ell:
      j_\ell\in\{1,2,\dots,(N^*)^d\}, 1\leq \ell\leq u\right\}.
$$
 If
 the support of $f\in L^p(\mathbb I^d)$ is contained in
$S:=\cup_{j\in \Upsilon_u}A_{  j}$ for a subset $\Upsilon_u$ of $\{1,2,\dots,(N^*)^d\}$ of cardinality at most $u$, we then say that $f$ is
$u$-sparse in $(N^*)^d$ partitions. Denote by  $Lip^{(N^*,u,r,c_0)}_{\mathbb I^d}$
the set of all    $f\in Lip^{(r,c_0)}_{\mathbb I^d}$ which are $s$-sparse in $(N^*)^d$
partitions. Let $L\sim \log m$, $d_1,d_2\sim
        \left(m\left(\frac{u}{(N^*)^d}\right)/\log m\right)^{1/(2r+d)}$, $d_3,\dots,d_L\sim\log m$.
{  If $m$ is large enough   to satisfy
$
 \frac{m}{\log m}\geq \tilde{C}_4\frac{(N^*)^\frac{2d+4r+2d}{(2r+d)p}}{u^\frac{1}{2r+d}},
$
then it can be easily deduced from \cite{Chui2020} that there exists a DSCN structure contained in $\mathcal N_{d_1,\dots,d_L}^{DFCN}$ such that
\begin{eqnarray}\label{sparse-1}
           &&C_6 m^{-\frac{2r}{2r+d}}\left(\frac{u}{(N^*)^d}\right)^{\frac{d}{2r+d}}
            \leq e(Lip^{(N^*,u,r,c_0)}_{\mathbb I^d},\Xi_p) \nonumber\\
            &\leq &
         \sup_{\rho\in\mathcal M(Lip^{((N^*)^d,u,r,c_0)},\Xi_p)} \mathbf E[\|f_{global}^{under}-f^*\|_{L^2_{\rho_X}}^2]\nonumber\\
            &\leq&
         C_7  \left(\frac{m}{\log m}\right)^{-\frac{2r}{2r+d}}  \left(\frac{u}{(N^*)^d}\right)^{\frac{2}{p}-\frac{2r}{2r+d}}, 
\end{eqnarray}
where $\tilde{C}_4, C_6,C_7$ are constants independent of $m$ and $f_{global}^{under}$ is an arbitrary global minimum of ERM on the DSCN. }   In (\ref{sparse-1}), $\left(\frac{m}{\log m}\right)^{-\frac{2r}{2r+d}}$ reflects the smoothness of $f^*$ and $\left(\frac{u}{(N^*)^d}\right)^{\frac{2}{p}-\frac{2r}{2r+d}}$ embodies the spatial sparseness of $f^*$.
{
As discussed above, given a sparsity level $u$ and the number of partitions $  (N^*)^d$, the size of data should satisfy
 $\frac{m}{\log m}\geq \tilde{C}_4\frac{(N^*)^\frac{2d+4r+2d}{(2r+d)p}}{u^\frac{1}{2r+d}}
$
to embody the spatially sparse feature of $f^*$. In particular, if the number of samples is smaller than the sparsity level $u$, it is impossible to develop learning schemes to realize the support of $f^*$.  Recalling the localized approximation property of deep ReLU nets \cite{Chui1994,Chui2020}, with the help of massive data, \eqref{sparse-1} shows that deep ReLU nets are capable of capturing the spatial sparseness, which is beyond the capability of shallow nets due to its lack of localized approximation \cite{Chui1994}.
We refer the readers to \cite{Chui2020} for more details about the above assertions.}
If $p=2$,  it can be found in (\ref{sparse-1}) that ERM on deep ReLU nets in the under-parameterized setting can achieve almost optimal generalization error rates.
The following theorem shows the existence of perfect   global minima of (\ref{target-optimization}) in the over-parameterized setting.

\begin{theorem}\label{Theorem:sparse-learning}
Let  $r,c_0>0$, $N^*\in\mathbb N$, $u\leq (N^*)^d$ and $m$ satisfy
$
 \frac{m}{\log m}  \succeq \frac{(N^*)^\frac{2d+4r+2d}{(2r+d)p}}{u^\frac{1}{2r+d}}.
$
{   If $L\succeq \log m$, $d_1,d_2\succeq m$ and $d_3,\dots,d_L\succeq\log m$,
then  there are are infinitely many $h\in\Psi_{d_1,\dots,d_L}$ such that
\begin{eqnarray}\label{good-global-sparse}
       &&C_6 m^{-\frac{2r}{2r+d}}\left(\frac{u}{(N^*)^d}\right)^{\frac{d}{2r+d}}
       \leq
     \sup_{\rho\in\mathcal M(Lip^{(N^*,u,r,c_0)},\Xi_2)} \mathbf E[\|h-f^*\|_{L^2_{\rho_X}}^2] \\
     & \leq&  2C_7
     \left(\frac{m}{\log m}\right)^{-\frac{2r}{2r+d}}\left(\frac{u}{(N^*)^d}\right)^{\frac{d}{2r+d}}.  
\end{eqnarray}
}
\end{theorem}

Theorem  \ref{Theorem:sparse-learning} shows that for spatially sparse regression functions, there are infinitely many global minima of (\ref{target-optimization}) in the over-parameterized setting achieving the almost optimal generalization error {  rates}. Besides the given three types of regression functions, we can provide similar results for numerous regression functions including the general composite functions \cite{Schmidt2020},  hierarchical interaction models \cite{Kohler2017} and piecewise smooth functions \cite{Imaizumi2018} by using the  same approach in this paper. {  In particular, to derive similar assertions as our theorems, it is sufficient to apply the proposed deepening scheme developed in Theorem \ref{Theorem:deepening} below on the corresponding generalization error estimates in \cite{Kohler2017,Imaizumi2018,Schmidt2020}. We remove the details for the sake of brevity.}

{  Based on the above theorems, we can derive the following corollary, which shows the versatility  of over-parameterized deep ReLU nets in regression.
\begin{corollary}\label{Corollary:versatility}
  Let $p\geq 2$ and $r,\gamma,c_0,c_0'>0$, $N\in\mathbb N$, $u\leq N^d$ and $m$ satisfy
$
 \frac{m}{\log m} \succeq \frac{N^\frac{2d+4r+2d}{(2r+d)p}}{u^\frac{1}{2r+d}}.
$
   If $L\succeq \log m$, $d_1,d_2\succeq m$ and $d_3,\dots,d_L\succeq\log m$,
then  there are are infinitely many $h\in\Psi_{d_1,\dots,d_L}$ such that \eqref{smooth-optimal-rate}, \eqref{good-global-additive}, and
\eqref{good-global-sparse}  hold simultaneously.
\end{corollary}

A crucial problem of deep learning is on how to specify the structure of deep nets for a {  given} learning task. It should be highlighted that in the under-parameterized setting, both the width and depth should be carefully tailored to avoid the bias-variance trade-off phenomenon, making the structures of deep nets for different learning tasks  quite different \cite{Kohler2017,Imaizumi2018,Schmidt2020,Chui2020,Han2020}. However, as shown in Corollary \ref{Corollary:versatility}, over-parameterizing succeeds in avoiding the structure selection problem of deep learning in the sense that there exist a unified over-parameterized structure of DFCN which contains perfect global minima for different learning tasks.
}


\section{Related Work and Discussions}\label{Sec.Related-work}

In this section, we review some related work on the generalization performance of deep ReLU nets and {  make some comparisons}  to highlight our novelty.
From the classical bias-variance trade-off principle, {  the}  over-parameterized setting
makes the deep nets model so flexible that its global minima   suffer  from the well known
over-fitting phenomenon {  \cite{Cucker2007} in the sense that they fit} the training data perfectly but {  fail} to predict new query points. Surprisingly,   numerical evidences \cite{Zhang2016} showed  that {  the over-fitting may not occur.}  This interesting phenomenon leads to a rethinking of the modern machine-learning practice and bias-variance trade-off.

{  The interesting result  in \cite{Belkin2018} is the first work}, to the best of our knowledge, to theoretically  study the generalization performance of   interpolation methods. In \cite{Belkin2018},  multivariate triangulations are constructed to interpolate data   and the generalization error bounds are exhibited as a function with respect to the dimension $d$, which shows that the interpolation method possesses good generalization performance when $d$ is large. After imposing  certain  structure constraints on the covariance matrices, \cite{Hastie2019,Bartlett2020}  derived tight generalization error bounds for  over-parameterized liner regression. In \cite{Muthukumar2020}, the authors revealed  several quantitative relations between linear interpolants and the structures of   covariance matrices  and then provided  a hybrid interpolating scheme whose generalization error was rate-optimal for sparse liner model with noise. Motivated by these results, \cite{Liang2020} proved that kernel ridgeless least squares possess  a good generalization performance for high dimensional data, provided the distribution $\rho_X$ satisfies certain restrictions. In \cite{Lin2020},   the generalization error of  kernel ridgeless least squares was proved to be   bounded by means of some differences of kernel-based integral operators.

It should be mentioned that there are some strict restrictions concerning the dimensionality, structures of covariance matrices  and  marginal distributions $\rho_X$ to guarantee the good generalization performance of interpolation methods in the existing literature. Indeed, it was proved in \cite{Kohler2019} that some restriction on the marginal distribution $\rho_X$ is necessary, without which there is a $\rho_X^*$ such that  all interpolation methods may perform extremely badly (see Lemma \ref{Lemma:bad-generalization}). However, the high dimensional assumption and structure constrains of covariance matrices are removable, since   \cite{Belkin2018} have already constructed a piecewise interpolation based on the well known Nadaraya-Watson estimator and derived   optimal generalization error {  rates}, without any restrictions on the dimension  and covariance matrices.

Compared with   {  the above} mentioned  existing work, there are mainly three novelties of our results. First, we aim to explain the over-fitting resistance phenomenon for deep learning rather than {  linear algorithms} such as linear regression and kernel regression. Due to the nonlinear nature of (\ref{target-optimization}), we rigorously prove the existence of  bad minima and perfect global ones. Therefore,   it is almost impossible to derive the same  results as  linear models \cite{Belkin2019,Muthukumar2020,Liang2020} for deep ReLU nets  to determine which conditions are sufficient to guarantee the perfect generalization performance of all global minima. Furthermore, our theoretical results coincide with the experimental phenomenon in the sense that global minima (with training error to be 0) with different parameters  frequently   have  totally different behaviors  in generalization.
Then, {  our results present essential advantages of running ERM on over-parameterized deep ReLU nets by means of proving the existence of deep ReLU nets possessing  the almost optimal generalization error rates, which is beyond the capability of shallow learning.} Finally, our results are established with mild conditions on   distributions and {  without any restrictions on the dimension of the input space}.

Another related work is \cite{Mucke2019}, where the authors discussed the generalization performance of interpolation methods based on histograms and also established the existence of bad and good interpolation neural networks. The main arguments of \cite{Mucke2019} and our paper are similar: there are global minima of  ERM on deep ReLU nets that can avoid  over-fitting. The main differences are as follows: 1) It is well known that an approximant or learner based on histograms suffers from the well known saturation problem in the sense that the approximation or learning rate cannot be improved further once the regularity (or smoothness) of the regression function achieves  certain level \cite{Yarotsky2017}. In particular, as shown in \cite{Yarotsky2017,Chui2020}, deep ReLU nets with two hidden layers can provide localized approximation but is difficult to approximate extremely smooth functions.  In our paper, we avoid this saturation phenomenon via deepening the networks and thus {  break}  through the bottleneck of the analysis in \cite{Mucke2019}. 2) We {  provide  detailed structures} of deep ReLU nets   and {  derive} the quantitative requirement of  the number of free parameters to guarantee the existence of global minima of ERM on deep ReLU nets, which is different from \cite{Mucke2019}. 3) More importantly, we devote to answering {  Problem \ref{prob:1}} via showing the optimal generalization error {  rates} and  the power of depth of some global minima of  ERM on deep ReLU networks. In particular, using the deepening approach, we prove that there {   exist infinitely many global minima } of ERM on over-parameterized deep ReLU nets that perform almost the same as the under-parameterized deep ReLU nets.

In summary, we provide an affirmative answer to {  Problem \ref{prob:1}  via providing several examples for perfect global minima of (\ref{target-optimization}).}  {  It would be interesting to
study the distribution of these perfect global minima and design feasible schemes to find   them. We will keep in studying  this topic and consider these two more challenging problems.}

\section{Numerical Examples}\label{Sec.numerical}
{  In this section, we conduct numerical simulations to support our theoretical assertions on the existence of benign overfitting of running ERM on over-parameterized deep ReLU nets. There are mainly four purposes of our simulations. In the first simulation, we aim to show the  relation between the generalization performance of global minima of (\ref{target-optimization}) and the number of parameters (or width) of deep ReLU nets. In the second one, we devote to verifying the over-fitting resistance of (\ref{target-optimization}) via showing the relation between the generalization error  and the number of algorithmic iterations (epoches). In the third one, we show the existence of    good and bad global minima of (\ref{target-optimization}).
Finally, we compare our learned global minima with some widely used learning schemes to show the learning performance of   (\ref{target-optimization}) on over-parameterized deep ReLU nets.}

\begin{figure*}[htb]
	\centering
	\includegraphics[scale=0.3]{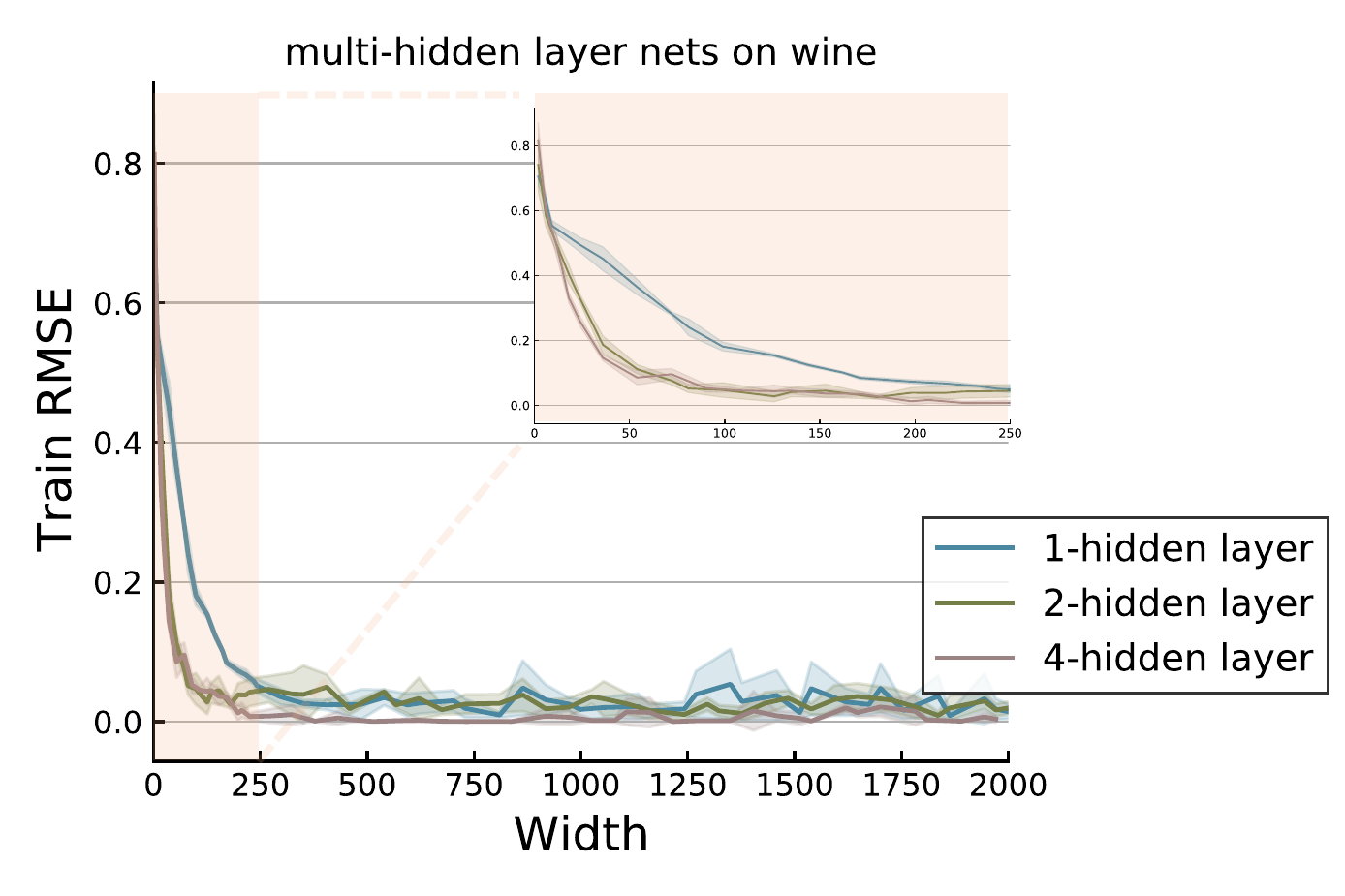}
	\includegraphics[scale=0.3]{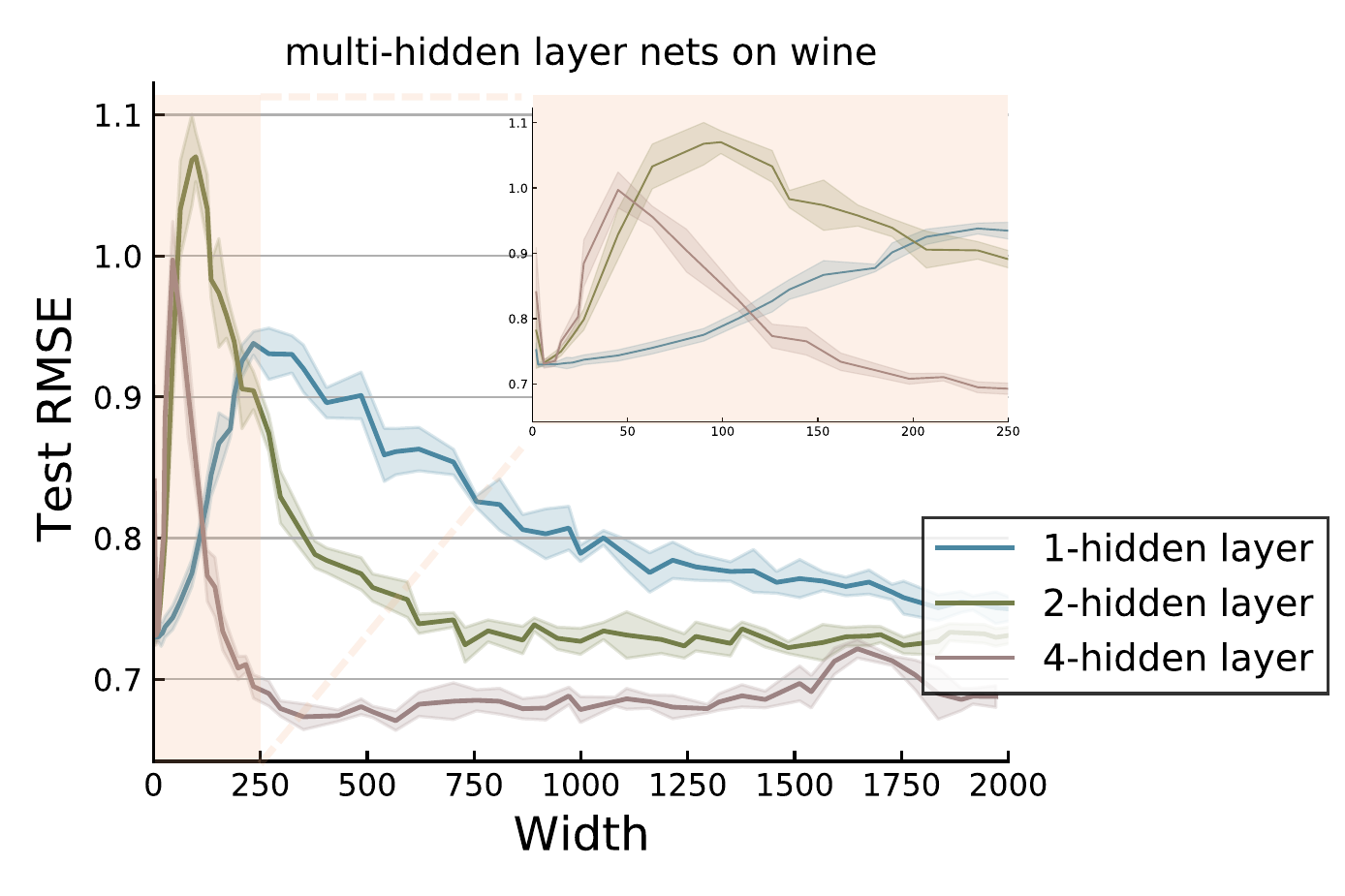}
	\includegraphics[scale=0.3]{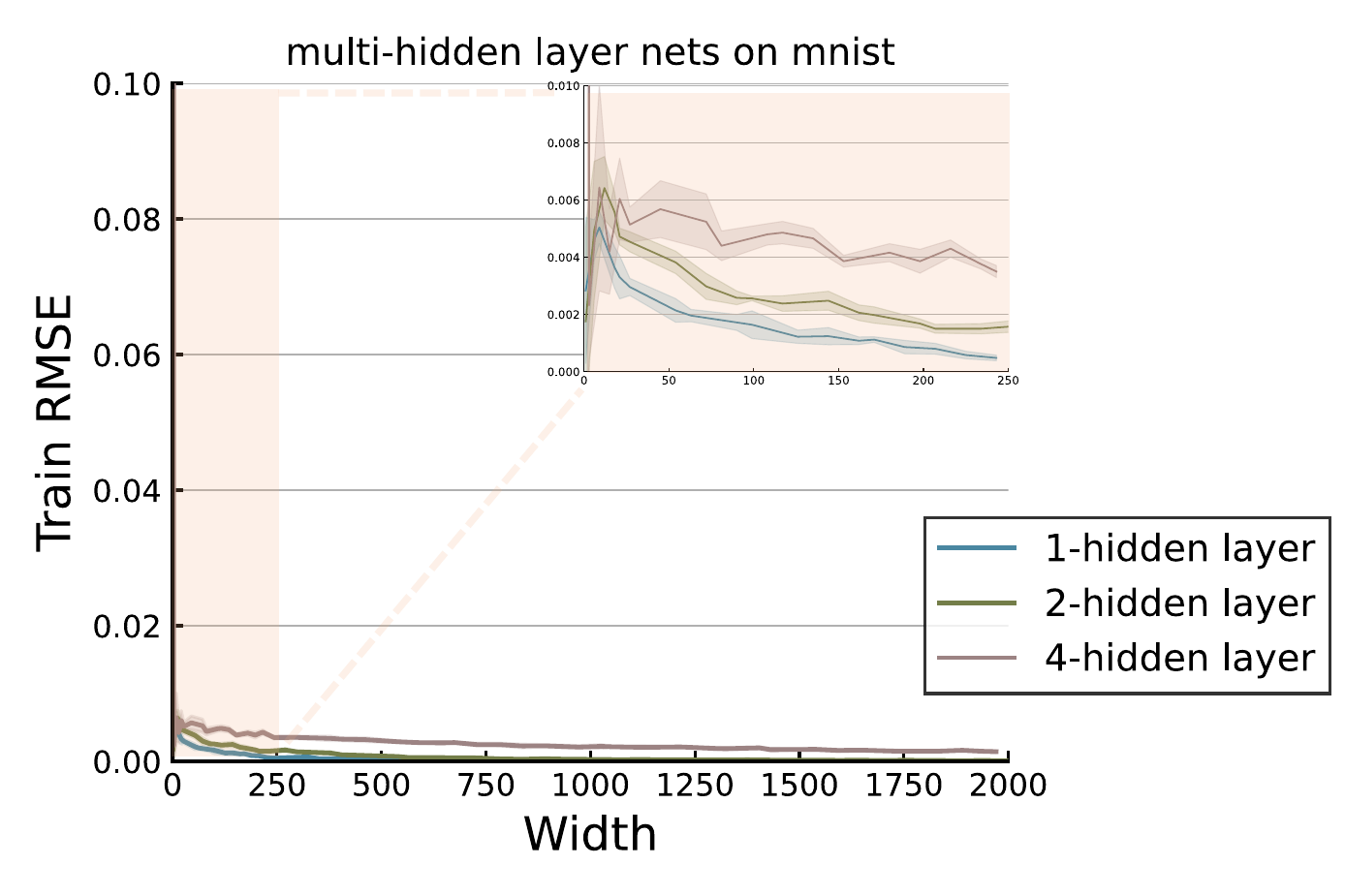}
	\includegraphics[scale=0.3]{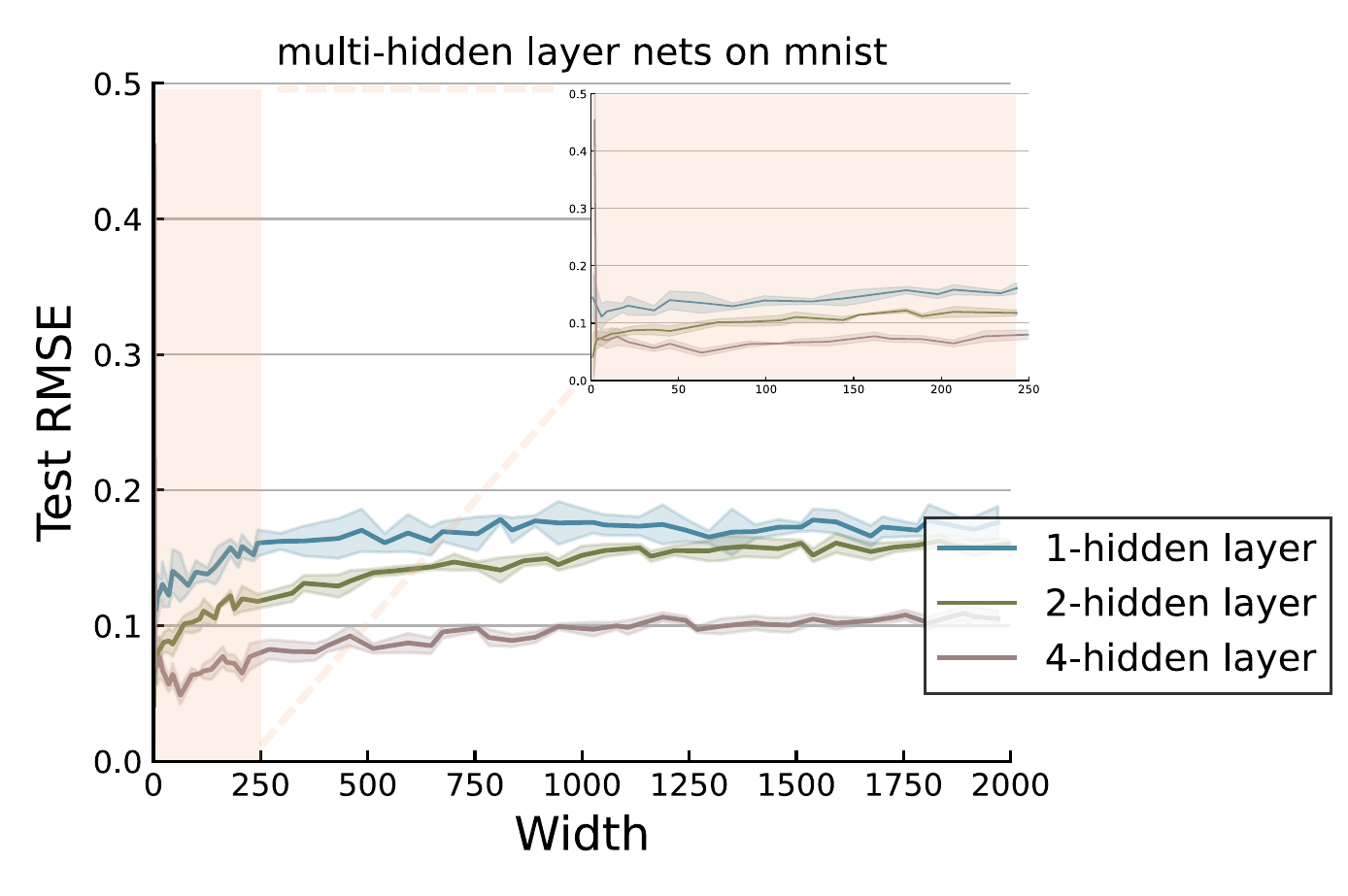}
	\caption{Relation between  training and testing errors and the width} \label{fig_1}
\end{figure*}

For these purposes, we adopt fully connected ReLU neural networks  with     $L$ hidden layers  and $k$ neurons paved on each layer. In all simulations, we set
$L\in \{1, 2, 4\}$ and $k\in\{1,\dots,2000\}.$ We use the well known Adam optimization algorithm on deep ReLU nets with step-size being constantly $0.001$ and initialization being the default PyTorch values. Without especial declaration, the training is  stopped after $50,000$ iterations.

we report our results on two real-world data numerical simulations on two datasets.
The first one is a Wine Quality dataset from UCI database. The Wine Quality dataset is related to red and white variants of the Portuguese ``Vinho Verde'' wine with $1599$ red and $4898$ white examples. We select white wine for experiments.
There are $12$ attributes in the data set:  fixed acidity, volatile acidity, citric acid, residual sugar, chlorides, free sulfur dioxide, total sulfur dioxide, density, pH, aulphates, alcohol and quality (score between $0$ and $10$). Therefore,  it can be viewed as a regression task on input data  of $11$ dimension. Regarding the preferences, each sample was evaluated by a minimum of three sensory assessors (using blind tastes), which graded the wine in a scale  ranging from $0$ (very bad) to $10$ (excellent). We sample $2/3$ data points as our training set and $1/3$ for testing.

 {The second dataset is MNIST. MNIST dataset is widely used in classification tasks. Here we follow \cite{Liang2020} to create a regression task using MNIST. MNIST inlcudes 70000 samples in total. Each sample includes a 28*28 dimensional feature and a target representing a digit ranging from 0 to 9. We randomly pick 291 samples with targeting digits equaling to 0 or 1, and separate 221 samples for training and 70 for testing. We label digit 0 as -1 and digit 1 as 1. Then we get a dataset with a $28\times28=784$ dimensional feature and a label -1 or 1. The regression task is then built. }

{\bf Simulation 1.}
In this simulation, we study the relation between the RMSE (rooted mean squared error) of test error and widths of deep ReLU nets with $L=1,2,4$.  {Our results are recorded via 5 independent single trials. The solid line is the mean value from these trials, and the shaded part indicates the deviation.} The numerical results are reported in Figure \ref{fig_1}.

Figure \ref{fig_1} presents the perfect global minima by exhibiting  the relation between training and testing RMSE and the number of free parameters. From Figure \ref{fig_1}, we  {  obtain} the following four observations: 1) The left figure shows that neural networks with more hidden layers are {  easier} to produce  exact interpolations of the training data. This coincides with the common consensus since more hidden layers with the same width involves much more free parameters; 2) For each depth, it can be found in the right figure that the testing curves exhibit an approximate doubling descent phenomenon as declared in \cite{Belkin2019} for linear models. It should be highlighted that such a phenomenon does not always exist for deep ReLU nets training and we only choose a good trend from several trails to demonstrate the existence of perfect global minima; 3) As the width (or capacity of the hypothesis space) increases, it can be found in the right figure that the testing error does not increase, exhibiting a totally different  phenomenon from the classical bias-variance trade-off. This shows that for over-parameterized deep ReLU nets, there  exist good global minima of (\ref{target-optimization}), provided the depth is appropriately selected; 4) It can be found that deeper ReLU nets  perform  better in generalization, which demonstrates the power of depth in tackling the Wine Quality data. All these verify our theoretical assertions in Section \ref{Sec.Noisy} and show that there {  exist} perfect global minima of (\ref{target-optimization}) to realize the power of deep ReLU nets.
%
%
%
%
%
%

{\bf Simulation 2.} In this simulation, we devote to numerically studying the role of iterations (epoches) in (\ref{target-optimization}) in  both under-parameterized and over-parameterized settings.  {We run ERM on  DFCNs with depth 4 and width $2, 40, 2000$ on Wine dataset and MNIST dataset. Since the number of training data is 3265(221 in MNIST dataset, resp.), it is easy to check that deep ReLU nets with depth 4 and widths 2 and 40 are under-parameterized,  while those with depth 4 and width 2000 is over-parameterized. }The numerical results are reported in Figure \ref{fig_2}.

\begin{figure*}[htb]
	\centering
	\includegraphics[scale=0.3]{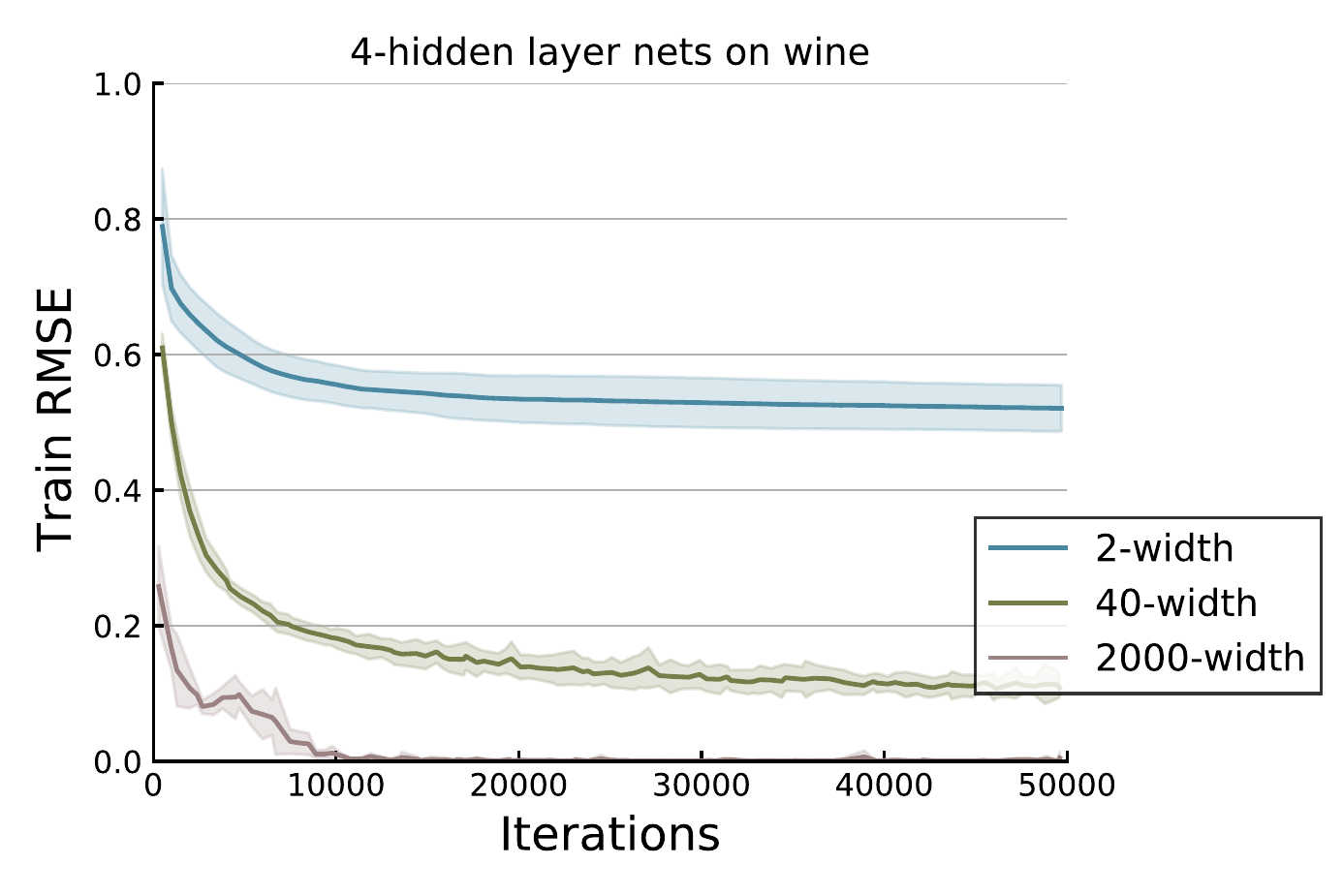}
	\includegraphics[scale=0.3]{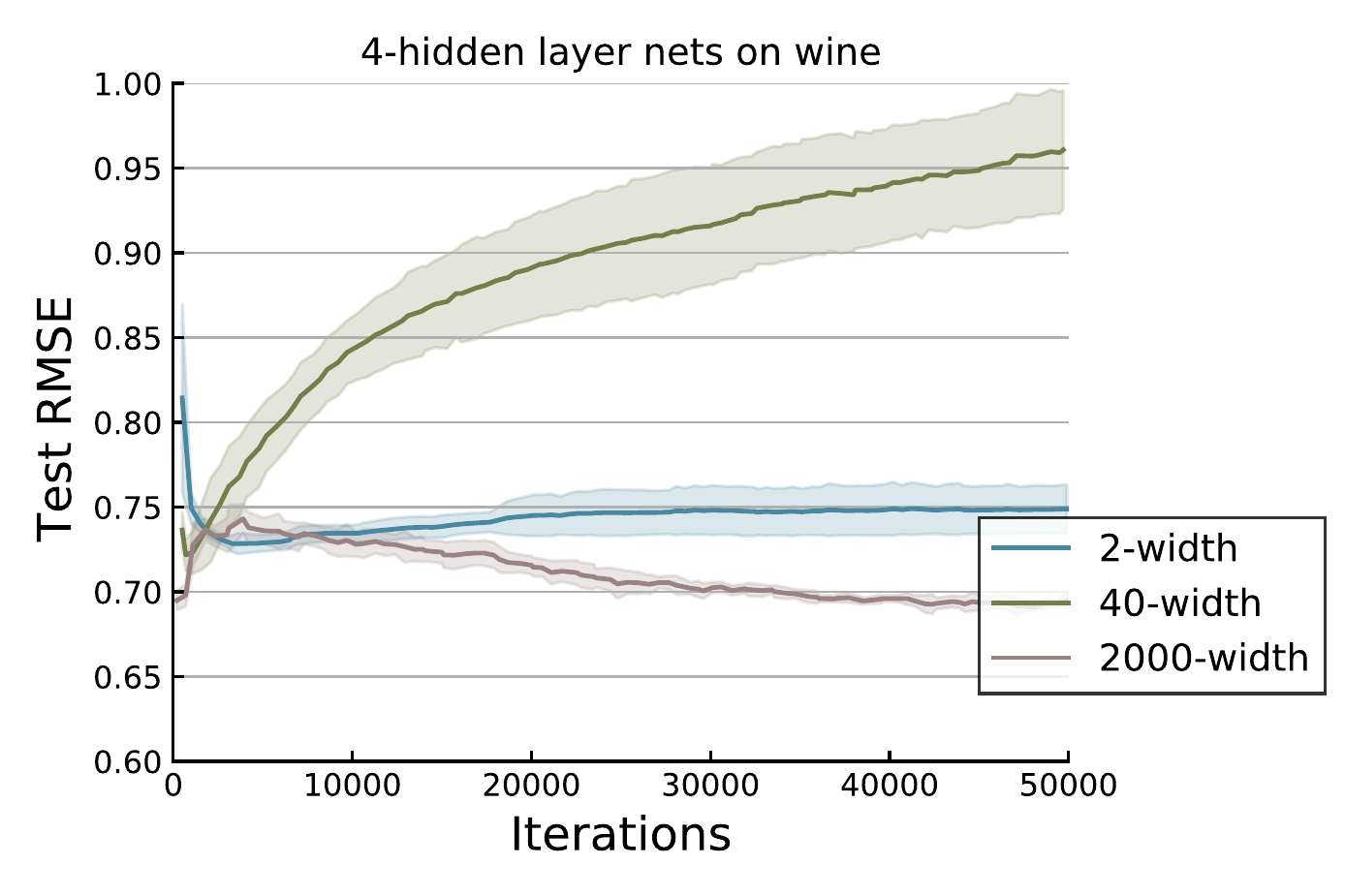}
	\includegraphics[scale=0.465]{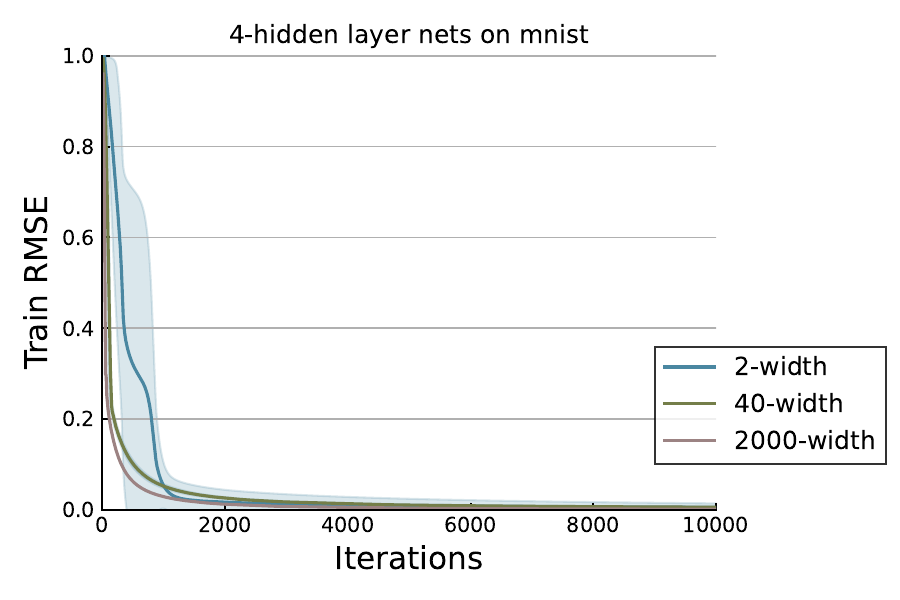}
	\includegraphics[scale=0.465]{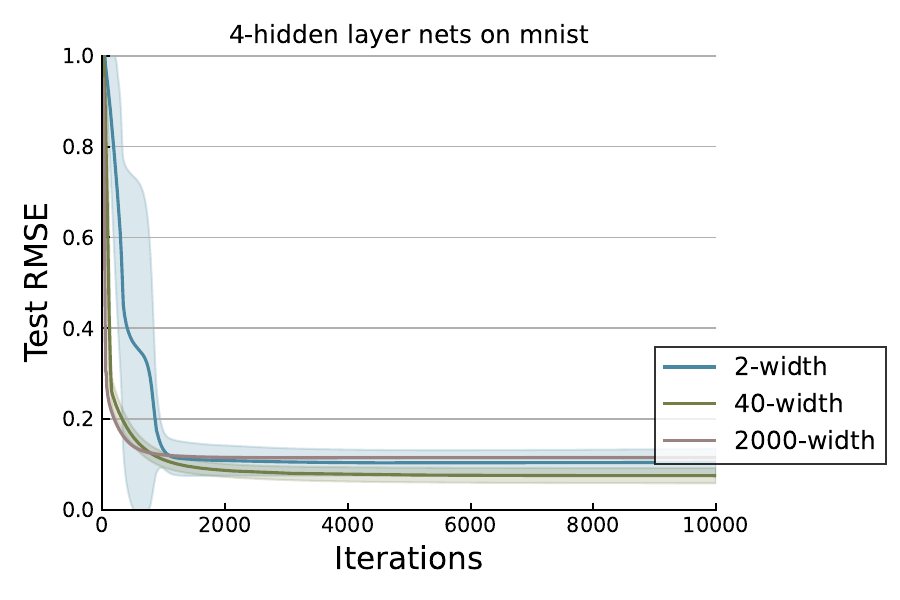}
	\caption{Relation between the training and testing errors and the number of iterations } \label{fig_2}
\end{figure*}

There are also three interesting findings exhibited in Figure \ref{fig_3}: 1) For under-parameterized ReLU nets, it is almost impossible to produce a global minimum acting as an exact interpolation of the data. However, for over-parameterized deep ReLU nets, running Adam with sufficiently many epoches attains a training error to be zero. Furthermore, after a specific value, the number of iterations does not affect the training error. This means that Adam converges to a global minimum of (\ref{target-optimization}) on over-parameterized deep ReLU nets; 2) The testing error for under-parameterized ReLU nets, exhibited in the right figure, behaves according to the classical bias-variance trade-off principle in the sense that the error firstly decreases with respect to the epoch and then increases after a specific value of epoches. Therefore, early-stopping is necessary to guarantee the good performance in this setting; 3) The testing error for over-parameterized ReLU nets is always non-increasing with respect to the epoch. This shows the over-fitting resistance of deep ReLU nets training and also verifies the existence of the perfect global minima of (\ref{target-optimization}) on over-parameterized deep ReLU nets. It should be highlighted the numerical {  result}  presented in Figure  \ref{fig_3} is also a single trial  selected from
numerous results, since we are concerned with the existence of perfect global minima. In fact, there are also numerous  examples for bad global minima of (\ref{target-optimization}).

%
%
%
%
%
%
%
{\bf Simulation 3}
 {In this simulation, we show that although there exist perfect global minima in over-parameterized settings, bad global minima can also be  found sometimes. We test the performance of deep ReLU nets with depth 4 and width 2000 on Wine dataset and MNIST dataset. It take different numbers of steps to converge to a good training performance on these two datasets. We trigger several runs with different learning rates and net parameter initializations, and pick good and bad global minima from two trials respectively. We report the numerical results in Figure\ref{fig_3}.}

\begin{figure*}[htb]
	\centering
	\includegraphics[scale=0.465]{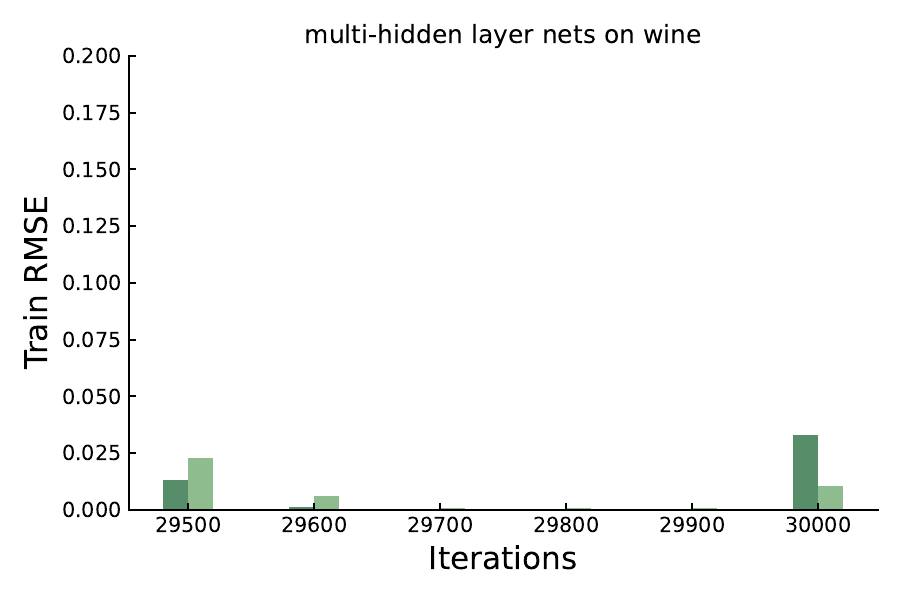}
	\includegraphics[scale=0.465]{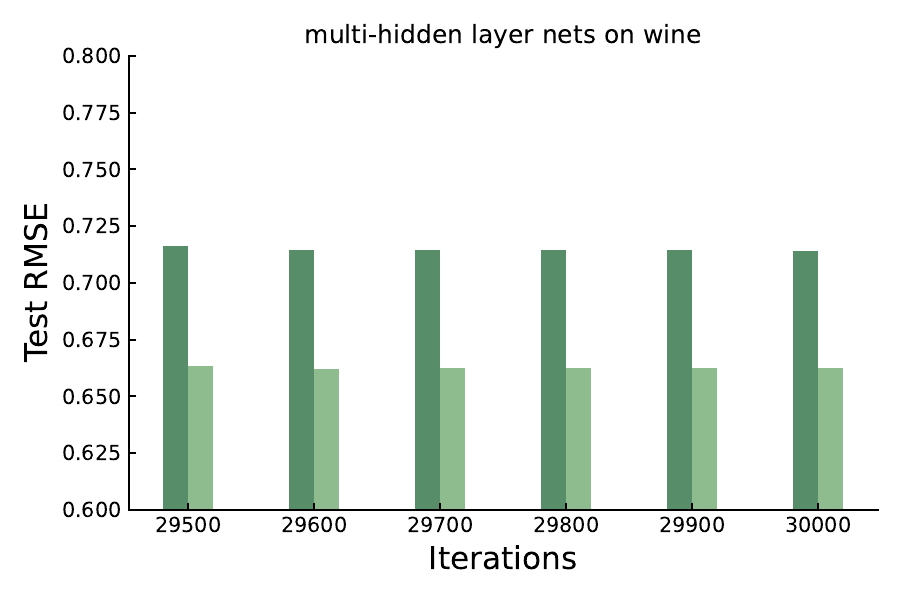}
	\includegraphics[scale=0.465]{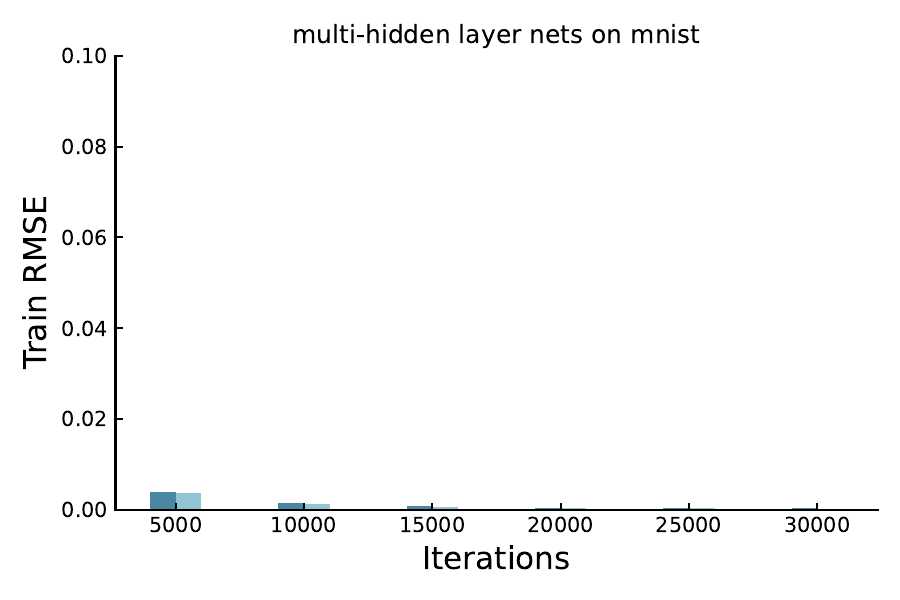}
	\includegraphics[scale=0.465]{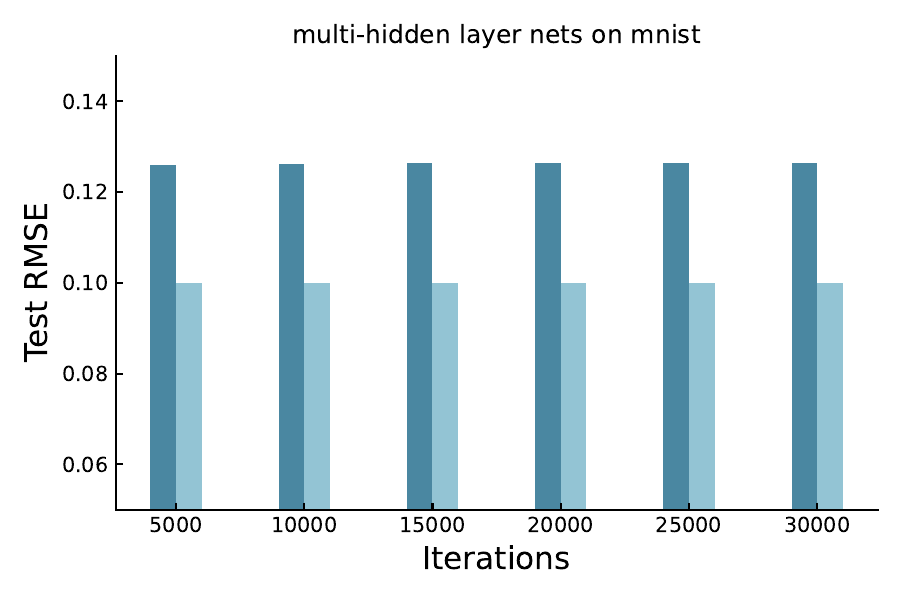}
	\caption{Comparison of good and bad global minima} \label{fig_3}
\end{figure*}

 {From Figure \ref{fig_3}, we find that  different global minima of Figure \ref{fig_3} perform totally differently in generalization,  though   the training loss both comes to 0. In particular, the testing errors of bad global minima can be much larger than those of good global minima. It should be mentioned that the bad interpolants in the above simulations are also derived from Adam. Therefore, their orders of testing errors are comparable with those of good interpolants. We highlight that this is due to the implementation of the ADAM algorithm rather than the model (\ref{target-optimization}). As far as the model is concerned, it can be shown in our next simulation that the orders of bad interpolants are also larger than those of good ones.}

{\bf Simulation 4}
 {In this simulation, we compare (\ref{target-optimization}) on over-parameterized deep ReLU nets with some standard learning algorithms including ridge regression (Ridge), support vector regression (SVR), kernel interpolation(KIR) and kernel ridge regression (KRR) to show that the numerical phenomenon exhibited in previous figures is not built upon sacrificing the generalization performance. For the sake of fairness,  we test various models under the same condition to our best efforts via tuning  hyper-parameters. In particular, we implement the referenced methods by using the standard  scikit-learn package. In the experiment with the wine dataset specifically, we use ridge regression with regularization parameter being $1$. In KRR, we use the Gaussian kernel with width  being $20$ and regularization parameter  being $0.0002$. KIR uses the Gaussian kernel with width being $5$ and regularization parameter being 0. SVR keeps the default sklearn hyper-parameters. In the experiment with the MNIST dataset, we change regularization parameter in ridge regression to 10 and keep other parameters fixed. In this simulation, we also use the deep ReLU nets with depth $4$ and width $2000$ and conduct 5 trials to record the average training and testing RMSE. }

 {
In addition,   we   construct a ReLU net that achieves 0 training error but performs extremely badly in the testing set to show that there really exist  extremely bad interpolants, just as Proposition \ref{Proposition:bad-interpolation} illustrated. We introduce some notations at first. Denote  $x_i=\left(x_i^{(1)}, \ldots, x_i^{(d)}\right)$. Define $T_{\tau, a, b}(t)=\frac{1}{\tau}\{\sigma(t-a+\tau)-\sigma(t-a)-\sigma(t-b)+\sigma(t-b-\tau)\}$, where $\sigma$ is the ReLU activation function and $\tau$ is a parameter which we  set to be small enough ($e^{-10}$ in this simulation). A feature input is expressed as $x=\left(x^{(1)}, \ldots, x^{(d)}\right)$.  We call the corresponding output of the constructed net (CN) as $N_{1, m, \tau}(x)$. Note that we drop duplicated samples when implementing CN. CN is constructed as bellow:
$$
N_{1, m, \tau}(x)=\sum_{i=1}^m y_i \sigma\left(\sum_{l=1}^d T_{\tau, x_i^{(l)}-\frac{1}{m^5}, x_i^{(l)}+\frac{1}{m^5}}\left(x^{(l)}\right)-(d-1)\right).
$$
More details of the construction of  $N_{1, m, \tau}$ can be found in the proof of Proposition \ref{Proposition:bad-interpolation}. We introduce $N_{1, m, \tau}$ in this simulation is to show that as an interpolant, $N_1$ performs quite poorly to show that there are extremely bad global minima of (\ref{target-optimization}). }
The numerical results are reported in  Table \ref{splitting}.

\begin{table}[htb]	
	\centering
	\scriptsize
	\setlength\tabcolsep{4.7pt}
	\begin{tabular}{cl|c|c}
		\toprule
		\multicolumn{4}{c}{\textbf{Wine Quality Data}} \\
		\midrule
		&  Methods       &  Train RMSE & Test RMSE   \\
		\midrule
		& Ridge & 0.534 & 0.735 \\
		& kernel interpolation & 0.000  & 13.031  \\
		& KRR   & 0.668 & 0.706  \\
		& SVR  & 0.628  & 0.696  \\
		\midrule \midrule
		&  \textbf{4-hidden layer DFCN (good case)} & \textbf{0.000}  & \textbf{0.668}  \\
		&  \textbf{CN (bad case)} & \textbf{0.000}  & \textbf{5.931}  \\
		\bottomrule
	\end{tabular}
	\begin{tabular}{cl|c|c}
		\toprule
		\multicolumn{4}{c}{\textbf{MNIST Data}} \\
		\midrule
		&  Methods       &  Train RMSE & Test RMSE   \\
		\midrule
		& Ridge & 0.056 & 0.304 \\
		& kernel interpolation & 0.000 & 0.135  \\
		& KRR   & 0.031 & 0.140  \\
		& SVR  & 0.073  & 0.154  \\
		\midrule \midrule
		&  \textbf{4-hidden layer DFCN (good case)} & \textbf{0.000}  & \textbf{0.097}  \\
		&  \textbf{CN (bad case)} & \textbf{0.000}  & \textbf{1.000}  \\
		\bottomrule
	\end{tabular}
	\caption{Comparison with other regression methods}
	\label{splitting}
\end{table}

{
There are four interesting observation in Table \ref{splitting}: 1) Learning schemes such as SVR, KRR and Ridge perform  stablely, since for both high-dimensional applications and low-dimensional simulations. The main reason is that a regularization term is introduced to balance the bias and variance for these schemes. As a result, the training error of these schemes are always non-zero; 2) Kernel interpolation performs well in high dimensional applications but fails to generalize well in low dimensional simulations. The main reason is that if $d$ is large, then the separation radius  $q_\Lambda$ is large \cite{Liang2020,Lin2020}, which in turn implies that the condition number of the kernel matrix is relatively small, making the kernel interpolation perform  well. However, if $d$ is small, the condition number of the kernel matrix is usually extremely large, making the prediction instable; 3) There exist deep ReLU nets exactly interpolating  the training data, leading to zero training error, but possessing an excellent generalization capability in yield small testing error,  implying that the obtain estimator is  a benign over-fitter for the data. Furthermore, it is shown in the table that the testing error of  over-parameterized deep ReLU nets is the smallest, demonstrating the power of depth as declared in our theoretical assertions in Section \ref{Sec.Noisy}; 4) There also exist deep ReLU nets interpolating the data but performing extremely badly in generalization, for both high-dimensional applications and low dimensional simulations.  All these findings verify our theoretical assertions that there are good global minima for ERM on over-parameterized deep ReLU nets but not all global minima are good. }

\section{Proofs}\label{Sec.proof}
In this section, we aim at proving our results stated in Section \ref{Sec.Noiseless} and Section \ref{Sec.Noisy}. The main novelty of our proof is a deepening scheme that produces an over-parameterized deep ReLU net (student network) based on a specific under-parameterized one (teacher network) so that the student network exactly interpolates the training data and possesses almost the same generalization performance as the teacher network.

\subsection{Deepening scheme for ReLU nets}

%

Given a teacher network $g$, the deepening scheme devotes to deepening    and widening it to produce a student network  $f$ that exactly interpolates the given data $D$ and possesses almost the same generalization performance as $g$.  The following theorem presents the deepening scheme in our analysis.

\begin{theorem}\label{Theorem:deepening}
Let {  $g_{n,L,U}$ be any deep ReLU nets with    $L$ layers, $n$ free parameters  and width not larger than $U\in\mathbb N$ satisfying $\|g_{n,L,U}\|_{L^\infty(\mathbb I^d)}\leq C^*$ for some $C^*>0$. If $\rho_X\in \Xi_p$ with $p\in[2,\infty)$,
then for any $\varepsilon>0$,
there exist infinitely many DFCNs $f_{D,n,L,U,g}$ of depth $\mathcal O(L+\log\varepsilon^{-1})$ and width $ \mathcal O(m+U+\log\varepsilon^{-1})$ such that
\begin{equation}\label{interpolation-1}
   f_{D,n,L,U,g}(x_i)=y_i, \qquad \ \forall i=1,\dots,m,
\end{equation}
and
\begin{equation}\label{interpolation-2}
       \|f_{D,n,L,U,g}- g_{n,L}\|_{L_{\rho_X}^2}\leq\varepsilon,
\end{equation}
where $\tilde{C}$ is a constant depending only on $d$. }
\end{theorem}

The deepening scheme developed in Theorem \ref{Theorem:deepening} implies that
all deep ReLU nets that have been verified to possess good generalization {  performances}  in the under-parameterized setting \cite{Imaizumi2018,Schmidt2020,Lin2020,Han2020} can be deepened to  corresponding  deep ReLU nets in the over-parameterized setting such that the deepened networks  exactly interpolate the given data and possess good generalization error bounds.


The main tools for the
proof of  Theorem \ref{Theorem:deepening} are the localized approximation property of deep ReLU nets developed in \cite{Chui2020} and the product gate property of deep ReLU nets proved in {  \cite{Yarotsky2017}}. Let us  introduce the first tool as follows.
  For $a,b\in\mathbb R$ with $ a<b $, define a trapezoid-shaped
function $T  _{\tau,a,b}$ with a parameter $0<\tau\leq 1$ as
\begin{eqnarray}\label{trapezoid function}
   T_{\tau,a,b}(t)&:=&\frac1\tau\big\{\sigma(t-a+\tau)-\sigma(t-a)\nonumber\\
    &-&
   \sigma(t-b)+\sigma(t-b-\tau)\big\}.
\end{eqnarray}
We   consider
\begin{eqnarray}\label{Def.N1}
   { \mathcal N}_{a,b,\tau}(x) := \sigma\left(\sum_{j=1}^dT_{\tau,a,b}(x^{(j)})-(d-1)\right).
\end{eqnarray}
The following lemma  proved in \cite{Chui2020} presents the localized approximation
property of ${ \mathcal N}_{a,b,\tau}$.

\begin{lemma}\label{Lemma:Local approximation}
Let  $ a<b$, $0<\tau\leq 1$ and ${  \mathcal N}_{a,b,\tau}$ be defined by (\ref{Def.N1}).
Then we have $0\leq {\mathcal N}_{a,b,\tau}(x)\leq 1$ for all $x\in\mathbb I^d$
and
\begin{equation}\label{Localized approximation}
     {\mathcal N}_{a,b,\tau}(x)=\left\{\begin{array}{cc}
     0,&\mbox{if}\ x\notin[a-\tau,b+\tau]^d,\\
     1,&\mbox{if}\ x\in [a,b]^d.
     \end{array}
     \right.
\end{equation}
\end{lemma}

The second tool, as shown in the following lemma, presents  the     product-gate property of deep ReLU nets \cite{Yarotsky2017}.
\begin{lemma}\label{lemma:product-gate-2.2}
 For any $\ell\in\{2,3,\dots,\}$ and $\nu\in
(0,1)$, there exists a  DFCN with ReLU activation functions $\tilde{\times}_{\ell,\nu}:\mathbb R^\ell\rightarrow\mathbb R$
with $\mathcal O\left(\ell\log\frac1\varepsilon\right)$ depth, $\mathcal O\left(\ell\log\frac1\varepsilon\right)$  width,  and {  free parameters bounded by $\mathcal O(\ell^\beta\nu^{-\beta})$ for some $\beta>0$}
 such that
$$
       |u_1u_2\cdots u_\ell-\tilde{\times}_{\ell,\nu}(u_1,\dots,u_\ell)|\leq
       \nu,\qquad \forall u_1,\dots,u_\ell\in[-1,1]
$$
and
$$
   \tilde{\times}_{\ell,\nu}(u_1,\dots,u_\ell)=0,\qquad \mbox{if}\quad u_j=0 \quad \mbox{for some}\quad j=1,\dots,\ell.
$$
\end{lemma}

With the above tools, we can prove Theorem \ref{Theorem:deepening} as follows.

\begin{proof}[Proof of Theorem \ref{Theorem:deepening}]
Let ${\mathcal N}_\tau={\mathcal N}_{-\tau,\tau,\tau/2}$ be given in Lemma \ref{Lemma:Local approximation} and $\tilde{\times}_{2,\nu}:\mathbb R^2\rightarrow\mathbb R$  in Lemma \ref{lemma:product-gate-2.2} with $\ell=2$. Then
it follows from \eqref{Localized approximation} that
\begin{equation}\label{proof.int-1}
      {\mathcal N}_{ \tau}(x-x_i)=\left\{\begin{array}{cc}
     0,&\mbox{if}\ x\notin x_i+[-3\tau/2,3\tau/2]^d,\\
     1,&\mbox{if}\ x\in x_i+[-\tau,\tau]^d.
     \end{array}
     \right.
\end{equation}
  Since $\|g_{n,L,U}\|_{L^\infty(\mathbb I^d)}\leq C^*$, we can
define a function $\mathcal N_{\tau,\nu,D,g}$ on $\mathbb R^d$ by
\begin{eqnarray}\label{good-deep-inter-noiseless}
    \mathcal N_{\tau,\nu,D,g}(x)&:=&\sum_{i=1}^m y_i  {\mathcal N}_{\tau}(x-x_i) \nonumber\\
   &+&C^*\tilde{\times}_{2,\nu}\left( \frac{g_{n,L,U}(x)}{C^*},  1-\sum_{i=1}^m {\mathcal N}_{\tau}(x-x_i)\right).
\end{eqnarray}
If $\tau<\frac{2q_\Lambda}{3\sqrt{d}}$, then  for any $j\neq i$, we have  from (\ref{proof.int-1}) that
$
      {\mathcal N}_{ \tau}(x_j-x_i)=0.
$
       Noting further ${\mathcal N}_{\tau}(x_i-x_i)=1$, we have for any $j=\{1,\dots,m\}$ that $\sum_{i=1}^m   {\mathcal N}_{\tau}(x_j-x_i)=1$ and
\begin{equation}\label{good-deep-inter-noiseless}
   \mathcal N_{\tau,\nu,D,g}(x_j)=\sum_{i=1}^m y_i  {\mathcal N}_{\tau}(x_j-x_i)=y_j.
\end{equation}
Moreover, for $i\neq j$ and any $x\in\mathbb R^d$,
$$
    \|x_i-x-(x_j-x)\|_2=\|x_i-x_j\|_2\geq \frac{2q_{\Lambda}}{\sqrt{d}}>3\tau
$$
implies $\mathcal N_\tau(x-x_j)=0.$ Hence $1-\sum_{i=1}^m\mathcal N_\tau(x-x_i)\in[0,1]$.
Therefore, Lemma \ref{lemma:product-gate-2.2} yields $\tilde{\times}_{2,\nu}\left( \frac{g_{n,L,U}(x_j)}{C^*},  1-\sum_{i=1}^m {\mathcal N}_{\tau}(x_j-x_i)\right)=0$. This implies
\begin{equation}\label{constr-interpolation}
     \mathcal N_{\tau, \nu,D,g}(x_j)=y_j,\qquad j=1,\dots,m.
\end{equation}
Define further a function $h_D$ on $\mathbb R^d$ by
$$
       h_D(x):=\sum_{i=1}^m y_i  {\mathcal N}_{\tau}(x-x_i)
   + g_{n,L,U}(x)\left(1-\sum_{i=1}^m {\mathcal N}_{\tau}(x-x_i)\right).
$$
It follows from Lemma \ref{lemma:product-gate-2.2} that
\begin{equation}\label{proof.int-2}
    |h_D(x)-\mathcal N_{\tau,\nu,D,g}(x)|\leq \nu,\qquad\ \forall x\in\mathbb I^d.
\end{equation}
If $x-x_i\notin[-3\tau/2,3\tau/2]^d$ for all $i=1,\dots,m$, then it follows from (\ref{proof.int-1}) that $\sum_{i=1}^m{\mathcal N}_{\tau}(x-x_i)=0$, which implies
$h_D(x)=g_{n,L}(x)$.
Hence,
\begin{eqnarray*}
     &&\|g_{n,L,U}-h_D\|_{L^p(\mathbb I^d)}^p
     =\int_{\mathbb I^d}|g_{n,L,U}(x)-h_D(x)|^pdx\\
     &\leq&\sum_{i=1}^m\int_{[x_i-3\tau/2,x_i+3\tau/2]^d}|g_{n,L,U}(x)-h_D(x)|^pdx
     \leq m(3\tau)^d 2^p(C^*)^p.
\end{eqnarray*}
This implies
$$
    \|g_{n,L,U}-h_D\|_{L^p(\mathbb I^d)}\leq 2C^*3^{d/p}m^{1/p}\tau^{d/p}.
$$
The above estimate together with (\ref{proof.int-2}) yields
\begin{eqnarray*}
     &&\|h_D-\mathcal N_{\tau, \nu,D,g}\|_{L^p(\mathbb I^d)}
     \leq
     \|h_D-\mathcal N_{\tau, \nu,D,g}\|_{L^p(\mathbb I^d)}+
      \|g_{n,L,U}-h_D\|_{L^p(\mathbb I^d)}\\
      &\leq&
      2^{d/p}\nu+2C^*3^{d/p}m^{1/p}\tau^{d/p}.
\end{eqnarray*}
Set $\nu = \varepsilon$ and $\tau\leq \min\{2q_\Lambda/(3\sqrt{d}),m^{-1/d}\varepsilon^{p/d}\}$. We obtain
\begin{equation}\label{error-relation}
    \|h_D-\mathcal N_{\tau,\nu,D,g}\|_{L^p(\mathbb I^d)}\leq C'\varepsilon,
\end{equation}
where $C':=2^{d/p}+2C^*3^{d/p}$. Denote by $\mathcal N^*(t)=\sigma(t)-\sigma(-t)=t$.
Recalling (\ref{good-deep-inter-noiseless}), we can define
\begin{eqnarray*}
     f_{D,n,L,U,g}&:=&\sum_{i=1}^my_i \overbrace{\mathcal N^*(\cdots\mathcal N^*}^{\mathcal O(L+\log \varepsilon^{-1})}(   {\mathcal N}_{\tau}(x-x_i) ))\\
   &+& C^*\tilde{\times}_{2,\nu}\left( \mathcal N^*\left(\frac{g_{n,L,U}(x)}{C^*}\right),  1-\sum_{i=1}^m {\mathcal N}_{\tau}(x-x_i)\right)
\end{eqnarray*}
with $\tau$ and $\nu$ as above so that the two items on the righthand side of $f_{D,n,g}$ have the same depth.
Then $f_{D,n,L,U,g}$ is a DFCN of   depth $\mathcal O(L+\log\varepsilon^{-1})$ and width $ \mathcal O(m+U+\log\varepsilon^{-1})$.
Noting further $\rho_X\in \Xi_p$, we then have
$\|f\|_{L^2_\rho(\mathbb I^d)}\leq D_{\rho_X}\|f\|_{L^p(\mathbb I^d)}$. This together with (\ref{error-relation}) yields
$$
    \|h_D-f_{D,n,L,U,g}\|_{L^2_{\rho_X} (\mathbb I^d)} \leq C'\varepsilon.
$$
Recalling that there are infinitely many $\tau$ satisfying  $\tau\leq \min\{2q_\Lambda/(3\sqrt{d}),m^{-1/d}\varepsilon^{p/d}\}$, then there are infinitely many such $f_{D,n,L,U,g}$.
Theorem \ref{Theorem:deepening}  is then proved  by scaling.
\end{proof}

\subsection{Proofs}
In this part, we prove our main results by using the proposed deepening scheme (for results concerning noisy data) and a functional analysis approach developed in \cite{Narcowich2004} (for results concerning noiseless data).
%

%
%
%

Firstly, we prove Proposition \ref{Proposition:bad-interpolation}
based on Lemma \ref{Lemma:Local approximation}.

\begin{proof}[Proof of Proposition \ref{Proposition:bad-interpolation}]
 If $y_i=0$, $i=1,\dots,m$, we can set $f_{D,L}(x)=0$. Then our conclusion
 naturally holds. Otherwise, we
define
\begin{equation}\label{bad-interpolation-nn}
   \mathcal N_{\tau,D}(x):=\sum_{i=1}^m y_i  {\mathcal N}_{-\tau,\tau,\tau/2}(x-x_i).
\end{equation}
If $\tau<\frac{2q_\Lambda}{3\sqrt{d}}$, then it follows from \eqref{good-deep-inter-noiseless} that
 $\mathcal N_{\tau,D}(x_i)=y_i$.
Since $\|f^*\|_{L^p(\mathbb I^d)}\geq c$, a direct computation yields
$$
    \|f^*-\mathcal N_{\tau,D}\|_{L^p(\mathbb I^d)}\geq\|f^*\|_{L^p(\mathbb I^d)}-\|\mathcal N_{\tau,D}\|_{L^p(\mathbb I^d)}\geq c-\|\mathcal N_{\tau,D}\|_{L^p(\mathbb I^d)}.
$$
But (\ref{proof.int-1}) together with (\ref{noiseless-setting}) and ${\mathcal N}_{ \tau}(x-x_i)\leq 1$ for any $x\in\mathbb I^d$ yields
\begin{eqnarray*}
    &&\|\mathcal N_{\tau,D}\|_{L^p(\mathbb I^d)}
    \leq  \sum_{i=1}^m \left(\int_{\mathbb I^d}  \left|f^*(x_i) {\mathcal N}_{-\tau,\tau,\tau/2}(x-x_i)\right|^pdx\right)^{1/p}\\
    &\leq& \sum_{i=1}^m |f^*(x_i)|\left(\int_{\mathbb I^d}  |{\mathcal N}_{-\tau,\tau,\tau/2}(x-x_i)|dx\right)^{1/p}\\
    &\leq&
    \sum_{i=1}^m |f^*(x_i)| \left(\int_{x:\|x-x_i\|_2\leq\frac{2\tau}{3}}dx\right)^{1/p}=
    \sum_{i=1}^m |f^*(x_i)| \left(\frac{3\tau}2\right)^{d/p}.
\end{eqnarray*}
{  Therefore, for
\begin{equation}\label{tau-selection}
    \tau<\min\left\{\frac{2q_\Lambda}{3\sqrt{d}}, \frac23\left(\frac{c}2\right)^{p/d}\left(\sum_{i=1}^m|y_i|\right)^{-p/d}\right\},
\end{equation}
 we have
$\|\mathcal N_{\tau,D}\|_{L_p(\mathbb I^d)}\leq c/2$, which yields
$$
    \|f^*-\mathcal N_{\tau,D}\|_{L_p(\mathbb I^d)}\geq c-c/2=c/2.
$$
Note further that $\mathcal N_{\tau,D}$ is a DFCN with 2 hidden layers with $d_1=4dm$ and $d_2=m$. Since $t=\sigma(t)-\sigma(-t)$, we can define $f_{D,d_1,d_2,\dots,d_L}$ iteratively by $f_{D,d_1,d_2}(x)$ satisfying $d_1\geq 4dm$, $d_2\geq m$ and
$$
     f_{D,d_1,d_2,\dots,d_{\ell+1}}= \sigma(f_{D,d_1,d_2,\dots,d_{\ell}})-\sigma(-f_{D,d_1,d_2,\dots,d_{\ell}}).
$$
Then $ {f_{D,d_1,d_2,\dots,d_L}}\in\Psi_{d_1,d_2,\dots,d_L,m}$.  Recalling the construction in \eqref{bad-interpolation-nn}, different $\tau$ corresponds to different neural networks and the above results hold hold for all $\tau$ satisfying \eqref{tau-selection}. Therefore, there are infinitely many deep ReLU nets formed as \eqref{bad-interpolation-nn}.
This completes the proof of Proposition \ref{Proposition:bad-interpolation}.}
\end{proof}

In the following, we construct some real-valued functions to feed $\tilde{\times}_{\ell,\nu}$ and derive a deep-net-based linear space  possessing good approximation properties.
For $t\in\mathbb R$, define
\begin{equation}\label{def.psi}
   \psi(t)=\sigma(t+2)-\sigma(t+1)-\sigma(t-1)+\sigma(t-2).
\end{equation}
Then
\begin{equation}\label{property psi}
     \psi(t)=\left\{\begin{array}{cc} 1,& \mbox{if}\ |t|\leq 1,\\
                                  0,&\mbox{if}\ |t|\geq 2,\\
                                  2-|t|,& \mbox{if}\  1<|t|<2.
                                  \end{array}\right.
\end{equation}
For $N\in\mathbb N$, $\alpha=(\alpha^{(1)},\dots,\alpha^{(d)})\in\mathbb N_0^d$, $|\alpha|=\alpha^{(1)}+\dots+\alpha^{(d)}\leq s$  and $\mathbf
j=(j_1,\dots,j_d)\in\{0,1,\dots,N\}^d,$  define
\begin{eqnarray}\label{set}
   \Phi_{N,\nu,s}&:=&\mbox{span}\left\{\tilde{\times}_{d+s,\nu} (\psi_{1,{\bf j}},\dots,
  \psi_{d,{\bf j}},\overbrace{x^{(1)}, \dots, x^{(1)}}^{\alpha^{(1)}},\right.\nonumber\\
  &&\left.\dots,\overbrace{x^{(d)},\dots, x^{(d)}}^{\alpha^{(d)}},\overbrace{1,\dots,1}^{s-|\alpha|})\right\},
\end{eqnarray}
where
\begin{equation}\label{def-psi-k}
    \psi_{k,{\bf j}}(x)=
    \psi\left(3N\left(x^{(k)}-\frac{j_k}N\right)\right).
\end{equation}
It is easy to see that for arbitrarily fixed $N,\nu,s$, $\Phi_{N,\theta,\nu,s}$ is a linear space of dimension at most  $d(N+1)^d\left(^{s+d}_{\ d}\right)$. {  Each element in $\Phi_{N,\nu,s}$ is a DFCN with depth
     $\mathcal O((s+d)\log\nu^{-1})$, $d_1= \mathcal O\left(d(N+1)^d\left(^{s+d}_{\ d}\right)\right)$ and $d_\ell=\mathcal O(\log\nu^{-1}).$} The approximation capability of the constructed linear space was deduced
 in \cite[Theorem 2]{Han2020} or \cite{Yarotsky2017}.

\begin{lemma}\label{Lemma:jackson}
Let  $\nu\in(0,1)$ and  $s,N\in\mathbb N_0$.
If $\nu=N^{-r-d}$ and
   $f\in Lip^{(r,c_0)}_{\mathbb I^d}$ with $0<r\leq s+1$, then there holds
\begin{equation}\label{Jackson11}
       \min_{h\in \Phi_{N,\nu,s}}\|f-h\|_{L^\infty(\mathbb I^d)}\leq C_1'c_0N^{-r/d},
\end{equation}
where $C_1' $ is a constant  depending only on $d$ and $r$.
\end{lemma}
To prove Theorem \ref{Theorem:interpolation-app-linear}, we need the following lemma proved in \cite{Narcowich2004}. It should be mentioned that for DFCN with larger depth and width,  the above assertions obviously hold. We use the following functional analysis tool that presents a close relation interpolation and approximation to minimize the depth.

\begin{lemma}\label{Lemma:Banach}
  Let $\mathcal U$ be a (possibly complex) Banach
space, $\mathcal V$ a subspace of $\mathcal U$, and $W^*$ a
finite-dimensional subspace of $\mathcal U^*$, the dual of $\mathcal
U$. If for every $w^*\in W^*$ and some $\gamma>1$, $\gamma$
independent of $w^*$,
$$
\|w^*\|_{\mathcal U^*}\leq\gamma\|w^*|_{\mathcal V}\|_{\mathcal
V^*},
$$
then for any $u\in\mathcal U$ there exists $v\in\mathcal V$ such that
$v$ interpolates $u$ on $W^*$; that is, $w^*(u)=w^*(v)$ for all
$w^*\in W^*$. In addition, $v$ approximates $u$ in the sense that
$\|u-v\|_{\mathcal U}\leq(1+2\gamma)\mbox{dist}_{_{\mathcal
U}}(u,\mathcal V)$.
\end{lemma}

To use the above lemma, we need to
  construct a special function to facilitate the proof. Our construction is motivated by \cite{Narcowich2004}.
  For any
$w^*=\sum_{j=1}^mc_j\delta_{x_j}\in W^*$, define
\begin{equation}\label{def.target-g}
    g_w(x)=\sum_{j=1}^m\mbox{sgn}(c_j)\left(1-\frac{\|x-x_j\|_2}{q_\Lambda}\right)_+,
\end{equation}
where
$\delta_{x_i}$ is the point evaluation  operator and $\mbox{sgn}(t)$ is the sign
function satisfying $\mbox{sgn} (t)=1$ for $t\geq0$ and $\mbox{sgn}
(t)=0$ for $t<0$. Then it is easy to see that $g_w$ is a continuous function. In the following, we present three important properties of $g_w$.

\begin{lemma}\label{Lemma:construction-of-target}
Let   $W^*=\mbox{span}\{\delta_{x_i}:i=1,\dots,m\} $. Then for any $w^*\in W^*$, there holds
(i) $ \|g_w\|_{L^\infty(\mathbb I^d)}=1, $\ \ \ (ii) $  w^*(g_w)=\|w^*\|,$ \ \ \ (iii) $  g_w\in Lip_{\mathbb I}^{(1,q_\Lambda^{-1})}$.
\end{lemma}

\begin{proof}  Denote   $A_j=B(x_j,q_{\Lambda})\cap \mathbb I^d$, where
$B(x_j,q_{\Lambda})$  is the ball with center $x_j$ and radius $q_{\Lambda}$. Then it follows from the definition of $q_\Lambda$ that $\dot{A}_j\cap \dot{A}_k=\varnothing$, where $\dot{A}_j=A_j\backslash \partial A_j$ and $\partial A_j$ denotes the boundary of $A_j$. Without loss of generality, we assume  $\mathbb I^d\backslash\bigcup_{j=1}^mA_j\neq\varnothing$. From (\ref{def.target-g}), we have $g_w(x)=0$ for $x\in \mathbb I^d\backslash\bigcup_{j=1}^mA_j\neq\varnothing$.  If there exist some $j\in\{1,\dots,m\}$ such that $x\in A_j$, then
$$
     g_w(x)=\mbox{sgn}(c_j)\left(1-\frac{\|x-x_j\|_2}{q_\Lambda}\right).
$$
So
$$
     |g_w(x)|=1-\frac{\|x-x_j\|_2}{q_\Lambda}\leq |g_w(x_j)|=1.
$$
Thus, $|g_w(x)|\leq 1$ for all $x\in\mathbb I^d$. Since
$|g_w(x_j)|=1$, $j=1,\dots,m$, we get $\|g_w\|_{L^\infty(\mathbb I^d)}=1$, which verifies (i). For $w^*\in W^*$, we have
\begin{eqnarray*}
   w^*(g_w)&=&\sum_{j=0}^mc_j\delta_{x_j}(g_w)
   =\sum_{j=0}^mc_jg_w(x_j)\\
   &=&\sum_{j=0}^mc_j\mbox{sgn}(c_j)
   =\sum_{j=0}^m|c_j|=\|w^*\|.
\end{eqnarray*}
Thus (ii) holds. The   remainder is to prove that $g_w$ satisfies (iii).
 We  divide the proof into four cases.
 
  If  $x,x'\in A_j$  for some $j\in\{1,\dots,m\}$, then it follows from (\ref{def.target-g}) that \begin{eqnarray*}
    &&|g_w(x)-g_w(x')|\\
    &=& \left|\mbox{sgn}(c_j)\left(1-\frac{\|x-x_j\|_2}{q_\Lambda}\right)- \mbox{sgn}(c_j)\left(1-\frac{\|x'-x_j\|_2}{q_\Lambda}\right)\right|\\
    &\leq&
     \frac{|\|x-x_j\|_2-\|x'-x_j\|_2|}{q_\Lambda}
     \leq
     \frac{\|x-x'\|_2}{q_\Lambda}.
\end{eqnarray*}

  If  $x,x'\in \mathbb I^d\backslash\bigcup_{j=1}^mA_j$, then the definition of $g_w$ yields   $g_w(x)=g_w(x')=0$, which implies
$|g_w(x)-g_w(x')|\leq \frac{\|x-x'\|_2}{q_\Lambda}.$

If $x\in A_j,x'\in A_k$ for $k\neq j$, then it is easy to see that for any $z\in \partial B(x_j,q_{\Lambda})$, $j=1,\dots,m$, there holds $g_w(z)=0$. Let $z_j,z_k$ be the intersections of the line segment   $xx'$ and $\partial B(x_j,q_{\Lambda})$, and the line segment  $xx'$ and $\partial B(x_k,q_{\Lambda})$, respectively. Then, we have $\|x-x'\|_2\geq \|x-z_j\|_2+\|x'-z_k\|_2$.
Since $x,z_j\in A_j$, $x',z_k\in A_k$ and $g_w(z_j)=g_w(z_k)=0$, we have
\begin{eqnarray*}
   &&|g_w(x)-g_w(x')|\leq |g_w(x)-g_w(z_j)|+|g_w(x')-g_w(z_k)|\\
   &\leq&
   \frac{\|x-z_j\|_2}{q_\Lambda}+\frac{\|x'-z_k\|_2}{q_\Lambda}
   \leq\frac{\|x-x'\|_2}{q_\Lambda}.
\end{eqnarray*}

If $x\in A_j$ for some $j\in\{1,\dots,m\}$ and  $x'\in \mathbb I^d\backslash\bigcup_{j=1}^mA_j$, then we take
  $z_j$ to  be the intersection of $\partial B(x_j,q_{\Lambda})$ and the line segment $xx'$.
Then $\|x-z_j\|_2\leq\|x-x'\|_2$ and
$$
   |g_w(x)-g_w(x')|=|g_w(x)|=|g_w(x)-g_w(z_j)|
   \leq \frac{\|x-z_j\|_2}{q_\Lambda}
   \leq \frac{\|x-x'\|_2}{q_\Lambda}.
$$

Combining all the above cases verifies   $g_w\in Lip_{\mathbb I}^{(1,q_\Lambda^{-1})}$. This completes the proof of Lemma \ref{Lemma:construction-of-target}.
\end{proof}

With the above tools,  we are in a position to prove Theorem \ref{Theorem:interpolation-app-linear}.

\begin{proof}[Proof of Theorem \ref{Theorem:interpolation-app-linear}]
Let $\mathcal U=C(\mathbb I^d)$, the space of continuous functions defined on $\mathbb I^d$,
 $\mathcal W^*=\mbox{span}\{\delta_{x_i}\}_{i=i}^m$  and $\mathcal V=\Phi_{N,\nu,s}$ in Lemma \ref{Lemma:Banach}.
For every $w^*\in\mathcal W^*$, we have $w^*=\sum_{i=1}^mc_i\delta_{x_i}$ for some $\vec{c}=(c_1,\dots,c_m)^T\in\mathbb R^m$. Without loss of generality, we assume $\|w^*\|=\sum_{i=1}^m|c_i|=1$.
Let $g_w$ be  defined by (\ref{def.target-g}). Then, it follows from Lemma \ref{Lemma:jackson} and Lemma \ref{Lemma:construction-of-target} that  there is some $h_g\in \mathcal V$  such that
$$
      \|g_w-h_g\|_{L^\infty(\mathbb I^d)}\leq C_1' q_\Lambda^{-1}N^{-1/d}.
$$
Let
$N\geq\left\lceil\left(\frac{(\gamma+1)C_1'}{(\gamma-1)q_\Lambda}\right)^d\right\rceil$ for some $\gamma>1$.
We have
$$
  \|g_w-h_g\|_{L^\infty(\mathbb I^d)} \leq\frac{\gamma-1}{\gamma+1}.
$$
This together with  (i) in Lemma \ref{Lemma:construction-of-target} yields
$$
   \|h_g\|_{L^\infty(\mathbb I^d)}\leq\frac{\gamma-1}{\gamma+1}+1=\frac{2\gamma}{\gamma+1}.
$$
  Since $w^*$ is a linear operator and (ii) in Lemma \ref{Lemma:construction-of-target} holds, there
holds
$$
   1=\|w^*\|=w^*(g_w)=w^*(g_w-h_g)+w^*(h_g).
$$
Hence, from
$$
   \|w^*(g_w-h_g)\|_{L^\infty(\mathbb I^d)}\leq\|w^*\|\|g_w-h_g\|_{L^\infty(\mathbb I^d)}\leq\frac{\gamma-1}{\gamma+1},
$$
we have
$$
   w^*(h_g)\geq1-|w^*(g_w-h_g)|\geq1-\frac{\gamma-1}{\gamma+1}=\frac2{\gamma+1}.
$$
Consequently,
\begin{eqnarray*}
\|w^*\|
 &=&
  1
 \leq
  \frac{\gamma+1}2w^*(h_g)
 \leq
\frac{\gamma+1}2\|w^*|_{\Phi_{N,\theta,\nu,s}}\|\|h_g\|_{L^\infty(\mathbb I^d)}\\
 &\leq&
\frac{\gamma+1}2\cdot\frac{2\gamma}{\gamma+1}\|\|w^*|_{\Phi_{N,\theta,\nu,s}}\|
 =
\gamma\|w^*|_{\Phi_{N,\theta,\nu,s}}\|.
\end{eqnarray*}
Setting $\gamma=2$, for any $f^*\in Lip_{\mathbb I^d}^{(r,c_0)}$, it follows from Lemma \ref{Lemma:Banach} and Lemma \ref{Lemma:jackson} that
there exists some  $h^*\in   \mathcal V=\Phi_{N,v,s}$ such that
$h^*(x_i)=f^*(x_i)$ and
$$
     \|h^*-f^*\|_{L^\infty(\mathbb I^d)}\leq 5\min_{h\in\mathcal V}\|h-f^*\|_{L^\infty(\mathbb I^d)}\leq 5C_1'c_0N^{-r/d}.
$$
{  Setting $\nu\sim N^{-r-d}$ and recalling   (\ref{def-psi-k}) and $t=\sigma(t)-\sigma(-t)$, $\Phi_{N,\nu,s}$ defined in (\ref{def-psi-k}) can be  regarded as the set of  DFCNs with depth
      $\mathcal O(\log N)$, $d_1=\mathcal O(N^d)$ and $d_\ell=\mathcal O(\log N)$. Noting that there are infinitely many $\nu\sim N^{-r-d}$, there are infinitely many DFCNs satisfying the above assertions.}
This completes the proof of Theorem \ref{Theorem:interpolation-app-linear}.
\end{proof}

The proofs of the other main results are simple by combining Theorem \ref{Theorem:deepening}  with existing results.

\begin{proof}[Proof of Theorem \ref{Theorem:smooth-learning}]
Setting the teacher network $g=f_{global}^{under}$ in (\ref{smooth-optimal-rate}), we obtain a student net $h$ based on Theorem \ref{Theorem:deepening} with  $\varepsilon=m^{-2r/(2r+d)}$. Then, Theorem \ref{Theorem:smooth-learning} follows from Theorem \ref{Theorem:deepening} and  (\ref{smooth-optimal-rate})  directly.
\end{proof}

\begin{proof}[Proof of Theorem \ref{Theorem:additive-learning}]
Based on Theorem \ref{Theorem:deepening}, we can set the teacher network to be the under-parameterized deep ReLU  $f_{global}^{under}$ in (\ref{additive-1}). Then, Theorem \ref{Theorem:additive-learning} follows from  Theorem  \ref{Theorem:deepening} and (\ref{additive-1}) directly.
\end{proof}

\begin{proof}[Proof of Theorem \ref{Theorem:sparse-learning}]
According to Theorem \ref{Theorem:deepening}, it is easy to obtain a student network $h$ based on the teacher network $g=f_{global}^{under}$ in (\ref{sparse-1}). Then, Theorem \ref{Theorem:sparse-learning} follows directly from Theorem \ref{Theorem:deepening} and (\ref{sparse-1}).
\end{proof}

\section*{Acknowledgement}
The work  of S. B. Lin is supported partially by   the  National Key R\&D Program of China (No.2020YFA0713900) and
   the National Natural Science Foundation of China
(No.62276209).  The work of Y. Wang is supported partially by the National Natural Science Foundation of China (No.11971374). The work of D. X. Zhou  is supported partially by the NSFC/RGC Joint Research Scheme [RGC Project No. N-CityU102/20 and NSFC Project No. 12061160462],
Germany/Hong Kong Joint Research Scheme  [Project No. G-CityU101/20], and
the Laboratory for AI-Powered Financial Technologies.

\end{document}